\documentclass[11pt]{article}
\usepackage[labelfont=bf]{caption}
\usepackage[utf8]{inputenc}
\usepackage[T1]{fontenc}
\usepackage{natbib}
\usepackage[colorlinks=true,citecolor=blue]{hyperref} 
\usepackage{url}            
\usepackage{booktabs}       
\usepackage{amsfonts}       
\usepackage{nicefrac}       
\usepackage{microtype}   
\usepackage{hhline}
\usepackage{makecell}
\usepackage{geometry}

\usepackage{comment}
\usepackage{lipsum}
\usepackage{setspace}
\usepackage{mathtools}
\usepackage{graphicx}
\usepackage{amssymb}
\usepackage{bbm}
\usepackage{amsthm}
\usepackage{xpatch}
\usepackage{mathrsfs}

\usepackage{multirow}
\usepackage{url}
\usepackage{array}
\usepackage{wrapfig}
\usepackage{tabularx}
\usepackage[normalem]{ulem} 
\usepackage{enumerate}
\usepackage{enumitem}
\usepackage[usenames,dvipsnames]{xcolor}
\usepackage{mathtools}
\mathtoolsset{showonlyrefs}
\usepackage{booktabs}       

\usepackage{footnote}

\newtheorem{lemma}{Lemma}
\newtheorem{theorem}{Theorem}

\newtheorem{assumption}{Assumption}

\newtheorem{proposition}{Proposition}

\newtheorem{condition}{Condition}

\def\vkh{$V_h^k$}

\def\vkh1{$V_{h+1}^k$}
\def\vkh1'{$V_{h+1}^k(s')$}

\def\ovh{$V_{h}^*(\cdot)$}

\def\ovh1{$V_{h+1}^*$}
\def\ovh1'{$V_{h+1}^*(s')$}

\def\tvkh{$ \tilde{V}_h^k $}

\def\tvkh1{ $\tilde{V}_{h+1}^k$ }
\def\tvkh1'{$\tilde{V}_{h+1}^k(s')$}

\def\pkhs'{$P_{s_h^k,a_h^k,s'} $}

\def\hpksa{$\hat{P}^k_{s,a}$}
\def\hpksa'{$\hat{P}^k_{s,a,s'}$}
\def\hpkh{$\hat{P}^k_{s_h^k,a_h^k}$}
\def\hpkh'{$ \hat{P}^k_{s_h^k,a_h^k,s'} $}

\newcommand{\expect}{\mathbb{E}}

\newcommand{\indict}{\mathbb{I}}
\newcommand{\states}{\mathcal{S}}
\newcommand{\trans}{P}
\newcommand{\actions}{\mathcal{A}}
\newcommand{\mdp}{M}

\usepackage{algorithm}
\usepackage{algpseudocode}

\newcommand{\poly}{\mathrm{poly}}

\usepackage{pifont}

\usepackage{hhline}
\usepackage{makecell}

\usepackage{algorithm}
\usepackage{algpseudocode}

\newcommand{\algoname}{SSTP}

\title{Nearly Minimax Optimal Reward-free Reinforcement Learning}
\author{
Zihan Zhang \\Tsinghua University \\ \texttt{zihan-zh17@mails.tsinghua.edu.cn} 
\and
Simon S. Du \\ University of Washington \\ \texttt{ssdu@cs.washington.edu}
\and
Xiangyang Ji \\Tsinghua University \\ \texttt{xyji@tsinghua.edu.cn}
		}
\begin{document}
	\maketitle
	\begin{abstract}
    We study the reward-free reinforcement learning framework, which is particularly suitable for batch reinforcement learning and scenarios where one needs policies for multiple reward functions.
This framework has two phases.
In the exploration phase, the agent collects trajectories by interacting with the environment without using any reward signal.
In the planning phase, the agent needs to return a near-optimal policy for arbitrary reward functions.
We give a new efficient algorithm, \textbf{S}taged \textbf{S}ampling + \textbf{T}runcated \textbf{P}lanning (\algoname), which interacts with the environment at most $O\left( \frac{S^2A}{\epsilon^2}\poly\log\left(\frac{SAH}{\epsilon}\right) \right)$ episodes in the exploration phase, and guarantees to output a near-optimal policy for arbitrary reward functions in the planning phase. 
Here, $S$ is the size of state space, $A$ is the size of action space, $H$ is the planning horizon, and $\epsilon$ is the target accuracy relative to the total reward.
Notably, our sample complexity scales only \emph{logarithmically} with $H$, in contrast to all existing results which scale \emph{polynomially} with $H$. 
Furthermore, this bound matches the minimax lower bound $\Omega\left(\frac{S^2A}{\epsilon^2}\right)$ up to logarithmic factors.

%
 
 Our results rely on three new techniques : 1) A new sufficient condition for the dataset to plan for an $\epsilon$-suboptimal policy
  ; 2) A new way to plan efficiently under the proposed condition using soft-truncated planning; 3) Constructing extended MDP to maximize the truncated accumulative rewards efficiently.
 
	\end{abstract}

\section{Introduction}\label{sec:intro}
Reinforcement learning (RL) studies the problem in which an agent aims to maximize its accumulative rewards  by interaction with an unknown environment. 
A major challenge in RL is \emph{exploration} for which the agent needs to strategically visit new states to learn transition and reward information therein. To execute efficient exploration, the agent must follow a well-designed adaptive strategy by which the agent is properly guided by the reward and transition information, other than the trivial random exploration. 
Provably algorithms have been proposed to help the agent visit new states efficiently with a fixed reward and transition model. 
See Section~\ref{sec:rel} for a review.

However, in various applications, it is necessary to re-design the reward function to incentivize the agent to learn new desired behavior 
 \citep{altman1999constrained,achiam2017constrained, tessler2018reward,miryoosefi2019reinforcement}.
To avoid repeatedly invoking the learning algorithm and interacting with the environment, it is desired to let the agent efficiently explore the environment \emph{without the reward signal} and collect data based on which the agent can compute a near-optimal policy for \emph{any} reward function.

The main challenge of this problem is that the agent needs to collect data that sufficiently covers the state space.
This problem was previous studied in \cite{brafman2002r,hazan2019provably,du2019provably}.
Recently, \citet{jin2020reward} formalized the setting, and named it reward-free RL.
In this setting, the agent first collects a dataset by interacting with the environment, and then is required to compute an $\epsilon$-optimal policy given any proper reward function.

\citet{jin2020reward} gave a formal theoretical treatment of this setting.
They designed a method which guarantees that by collecting $O\left(\left(\frac{S^2AH^3}{\epsilon^2}+\frac{S^4AH^6}{\epsilon}\right)\poly\log\left(SAH/\epsilon\right)\right)$  episodes, the agent is able to output an $\epsilon$-optimal policy, where $S$ is the number of states, $A$ is the number of actions, and $H$ is the planning horizon. \footnote{Because we consider reward function satisfying the total reward bounded by $1$ setting throughout this paper, we rescale the error $\epsilon$  to $\epsilon H$ in their bound.}
They also provided an $\Omega\left(\frac{S^2A}{\epsilon^2}\right)$ lower bound.
Recently,   \citet{kaufmann2020adaptive, menard2020fast} gave tighter sample complexity bound.
 We refer the readers to table \ref{tab:comparisons} for more details.




Unfortunately, it remains open what is the fundamental limit of the sample complexity of reward-free RL. 
In particular, compared to the $\Omega(\frac{S^2A}{\epsilon^2})$ lower bound, 
all existing upper bounds have a \emph{polynomial} dependence on $H$.
The gap between upper and lower bound can be huge for environments with a long horizon.
Conceptually, this gap represents that we still lack understanding on whether long horizon imposes significant hardness in reward-free RL.

\subsection{Our Contribution}
In this work, we break the $\poly\left(H\right)$-dependency barrier.
We design a new algorithm, \textbf{S}taged \textbf{S}ampling + \textbf{T}runcated \textbf{P}lanning (\algoname), which enjoys the following sample complexity guarantee.

%

\begin{theorem}\label{thm1}
For any $\epsilon,\delta \in (0,1)$, there exists an algorithm (SSTP, Algorithm~\ref{alg:main}) which can compute an $\epsilon$-optimal policy for any reward function that is non-negative and totally bounded by $1$, after 
$O\left( \left(\frac{SA }{\epsilon^2}\left(S+\log\left(\frac{1}{\delta}\right)\right)\right)\poly\log\left(SAH/\epsilon\right)\right)$ 
episodes of exploration with probability $1-\delta$.
\end{theorem}
The key significance of our theorem is that we match the lower bound of $\Omega\left(\frac{SA\left(S+\log1/\delta\right)}{\epsilon^2}\right)$ up to logarithmic factors on $S,A,H,1/\epsilon$.
\footnote{The original bound is for the case the total reward is bounded by $H$, and here we scale down the total reward by a factor of $H$.}
\footnote{
	\cite{jin2020reward,kaufmann2020adaptive,menard2020fast} studied reward-free exploration on the non-stationary episodic MDP (i.e., the transition model depends on the level $h \in [H]$), where the lower bound of sample complexity is at least linear in $H$ because of the complexity of MDP is larger~\citep{jin2018q,zhang2020almost}.  In this paper,
	 we consider the episodic MDP with stationary transition, that is, the transition model is independent of the level. 
	This is often considered to be a more realistic model than the non-stationary transition model.
}
Importantly, our bound only depends logarithmically on the planning horizon $H$.
This is an exponential improvement over existing results, and demonstrates that long horizon poses little additional difficulty for reward-free RL.
Furthermore, our bound only requires the reward to be totally bounded (cf. Assumption~\ref{asmp:total_bound}), in contrast to uniformly bounded (cf. Assumption~\ref{asmp:uniform}), which is assumed in previous works.
See Section~\ref{sec:rel} for more discussions.








\section{Related Work}\label{sec:rel}

\begin{savenotes}
	\begin{table}[t]
		\centering
		\resizebox{1\columnwidth}{!}{%
			\renewcommand{\arraystretch}{2}
			\begin{tabular}{ |c|c|c|c|c|}
				\hline
				\textbf{Algorithm}  & \textbf{Sample Complexity}& \makecell{\textbf{Non-unif.}\\ \textbf{Reward}}&\makecell{\textbf{Log H}}  \\
				\hline
				\hhline{|=|=|=|=|=|}

				\makecell{\textsc{RF-RL-Explore} \\
					\cite{jin2020reward} } & $\tilde{O}\left(\frac{H^5S^2A}{\epsilon}\log^3(\frac{1}{\delta})+ \frac{H^3S^2A}{\epsilon^2}\log(\frac{1}{\delta}) \right)$ & No & No   \\
				\hline
				
				\makecell{\textsc{RF-UCRL}  \\
					\cite{kaufmann2020adaptive}
				}         &   $\tilde{O}\left( \frac{H^2SA}{\epsilon^2}(\log(\frac{1}{\delta})+S) \right)$ & No & No\\
				\hline
								\makecell{\textsc{RF-Express} \\
					\cite{menard2020fast}
				}         &   $\tilde{O}\left( \frac{HSA}{\epsilon^2}(\log(\frac{1}{\delta})+S) \right)$ & No & No\\
				\hline
				\makecell{ \algoname	\\ This Work	} &  $\tilde{O}\left(\frac{SA}{\epsilon^2}(\log(\frac{1}{\delta})+S)\right)$
				& Yes & Yes \\
				\hline 
			    \makecell{Lower Bound} &  $\Omega\left( \frac{SA}{\epsilon^2}(\log(\frac{1}{\delta})+S) \right)$ & - & -\\
			    \hline
			\end{tabular}
		}
		\caption{
			Sample complexity comparisons for state-of-the-art episodic RL algorithms.
			See Section~\ref{sec:rel} for discussions on this table.
			$\widetilde{O}$ omits logarithmic factors on $S,A,H,1/\epsilon$ but not $1/\delta$.
			\textbf{Sample Complexity}: number of episodes to find an $\epsilon$-suboptimal policy.
			\textbf{Non-unif. Reward}: Yes means the bound holds under Assumption~\ref{asmp:total_bound} (allows non-uniformly bounded reward), and No means the bound only holds under Assumption~\ref{asmp:uniform}.
			\textbf{Log H}: Whether the sample complexity bound depends logarithmically on $H$ instead of polynomially on $H$.
			\label{tab:comparisons}
		}
	\end{table}
\end{savenotes}


\paragraph{Reward-dependent exploration}
In reward-dependent exploration, the agent aims to learn an $\epsilon$-optimal policy under a fixed reward.
Some papers assumed there is a generative model which can be queried to provide a sample for any state-action pair $(s,a)$ \citep{kearns1999finite,azar2013minimax,sidford2018near,agarwal2019model,li2020breaking}, and the sample complexity is defined as the number of queries needed to compute an $\epsilon$-optimal policy. In the online setting \citep{brafman2002r,kakade2003sample,dann2015sample,dann2017unifying,dann2019policy,azar2017minimax,jin2018q,zanette2019tighter,kaufmann2020adaptive,zhang2020almost, wang2020long,zhang2020reinforcement}, the agent starts from a fixed initial distribution in each episode, and collects a trajectory by interacting with the environment. Then the sample complexity is given by the number of episodes that are necessary to learn an $\epsilon$-optimal policy.
\footnote{There are several different measurements for the online reward-dependent exploration, we refer readers to \cite{dann2019policy} for more details.}
The state-of-the-art result by \citet{zhang2020reinforcement} requires $\widetilde{O}\left(\frac{SA}{\epsilon^2} + \frac{S^2A}{\epsilon}\right) $ number of episodes.

. 





\paragraph{Reward Assumption and Dependency on $H$} 
 For the reward, the widely adopted assumption is $r_h\in [0,1]$ for all $h\in [H]$, which implies the total reward $\sum_{h=1}^H r_h\in [0,H]$. 
 However, as argued in \cite{kakade2003sample,jiang2018open}, the characterization of sample complexity should be independent of the scaling, i.e., the target suboptimality $\epsilon \in (0,1)$ should be a \emph{relative} quantity to measure the performance of an algorithm.
 To this end, we need to scale the total reward within $[0,1]$.
Then the assumption becomes:
 \begin{assumption}[Uniformly Bounded Reward]\label{asmp:uniform}
 	The reward satisfies that $r_h\in [0,1/H]$ for all $h\in [H]$.
 \end{assumption}
 
 Compared to Assumption~\ref{asmp:uniform}, the totally-bounded reward assumption  (Assumption~\ref{asmp:total_bound}) is more general. Therefore, any upper bound under
 Assumption~\ref{asmp:total_bound} implies an upper bound under Assumption~\ref{asmp:uniform}.
 In the view of practice, because environments under Assumption~\ref{asmp:total_bound} can have high one-step reward,
 it is more natural to consider Assumption~\ref{asmp:total_bound} in environments
 with sparse rewards, such as the Go game, which are often considered to be puzzling.
 In the view of theoretical basis, it is more complicated to design efficient algorithms under Assumption~\ref{asmp:total_bound} due to the global structure of the reward.
 \footnote{Under Assumption~\ref{asmp:total_bound}, the reward still satisfies $r_h \in [0,1]$, so if an algorithms enjoys an sample complexity bound under Assumption~\ref{asmp:uniform}, scaling up this bound by an $H^2$ for PAC bound,  one can also obtain a bound under Assumption~\ref{asmp:total_bound}.
 However, this reduction is highly suboptimal in terms of $H$, so when comparing with existing results, we display their original results and add a column indicating whether the bound is under Assumption~\ref{asmp:total_bound} or Assumption~\ref{asmp:uniform}.}

Recent work \citep{wang2020long,zhang2020reinforcement} made essential progress in reward-dependent exploration  under Assumption \ref{asmp:total_bound}, and obtained sample complexity bounds that only scale \emph{logarithmically} with $H$. \cite{wang2020long} proved a sample complexity bound of $\tilde{O}\left(\frac{S^5A^4\log(\frac{1}{\delta})}{\epsilon^3} \right)$ despite suffering exponential computational cost, and later \cite{zhang2020reinforcement} achieved a nearly sharp sample complexity bound of $\tilde{O}\left( \frac{SA\log(\frac{1}{\delta})}{\epsilon^2}+S^2A\log(\frac{1}{\delta})\right)$ with polynomial computational cost. 
We use some technical ideas from \cite{zhang2020reinforcement} (cf. Section~\ref{sec:sampling}).
However, because of the different problem settings, we still need additional efforts to establish near-tight sample complexity bound for reward-free exploration.

\paragraph{Reward-Free RL}
The main algorithm in \cite{jin2020reward} assigns only non-zero reward for each state at every turn, and utilizes a regret minimization algorithm \textsc{EULER} \citep{zanette2019tighter} to visit each state as much as possible. 
Since their algorithm only learns one state each time, their sample complexity bound is not tight with respect to $H$.  \citet{kaufmann2020adaptive} proposed \textsc{RF-UCRL} to achieve sample complexity of $\tilde{O}\left( \frac{SAH^2}{\epsilon^2}(S+\log(\frac{1}{\delta})) \right)$ by building upper confidence bounds for any reward function and any policy, and then taking the greedy policy accordingly. The later work by \citet{menard2020fast} constructed an exploration bonus of  $\frac{1}{n(s,a)}$ instead of the classical exploration bonus of $\frac{1}{\sqrt{n(s,a)}}$, where $n(s,a)$ is the visit count of $(s,a)$. Based on the novel bonus, they achieve sample complexity of $\tilde{O}(  \frac{SAH}{\epsilon^2}(S+\log(\frac{1}{\delta}))$.
Recently, these results have been extended to linear function approximation settings~\citep{wang2020reward,zanette2020provably}.
Reward-free is also related to another setting, reward-agnostic RL,  in which $N$ reward functions are considered in the planning phase.
\citet{zhang2020task} provided an algorithm which achieves $\tilde{O}\left(\frac{H^3SA\log(N)\log(1/\delta)}{\epsilon^2}\right)$ sample complexity.



\section{Preliminaries}\label{sec:pre}

\paragraph{Notations.}
Throughout this paper, we define $[N]$ to be the set $\{1, 2, \ldots, N\}$ for $N\in \mathbb{Z}_{+}$.
We use $\mathbb{I}[\mathcal{E}]$ to denote the indicator function for an event $\mathcal{E}$, i.e., $\indict[\mathcal{E}] = 1$ if $\mathcal{E}$ holds and $\indict[\mathcal{E}] = 0$ otherwise.
For notational convenience, we set $\iota  = \ln(2/\delta)$ throughout the paper.  
For two $n$-dimensional vectors $x$ and $y$, we use $xy$  to denote  $x^{\top}y$, use $\mathbb{ V}(x ,y) = \sum_{i}x_i y_i^2 $, and use $x^2$ to denote the vector $[x_1^2,x_2^2,...,x_n^2]^{\top}$ for $x = [x_1,x_2,...,x_n]^{\top}$. 
For two vectors $x,y$, $x \geq y$ denotes $x_i \geq y_i$ for all $i \in [n]$ and $x \leq y$ denotes $x_i \le y_i$ for all $i \in [n]$. We use $\textbf{1}$ to denote the $S$-dimensional vector $[1,...,1]^{\top}$ and $\textbf{1}_{s}$ to denote the $S$-dimensional vector $[0,...,1,...,0]^{\top}$ where the only non-zero element is in the $s$-th dimension.

\paragraph{Episodic Reinforcement Learning.}
We first describe the setting for standard episodic RL.
A finite-horizon Markov Decision Process (MDP) is a tuple $\mdp =\left(\states, \actions, \trans ,R, H, \mu\right)$.
$\states$ is the finite state space with cardinality $S$.
$\actions$ is the finite action space with cardinality $A$.
$\trans: \states \times \actions \rightarrow \Delta\left(\states\right)$ is the transition operator which takes a state-action pair and returns a distribution over states. For $h=1,2,...,H$,
$R_h : \states \times \actions \rightarrow \Delta\left( \mathbb{R} \right)$ is the reward distribution with a mean function $r_h:\states \times\actions \rightarrow \mathbb{R}$.
$H \in \mathbb{Z}_+$ is the planning horizon  (episode length).
 $\mu \in \Delta\left(\states\right)$ is the initial state distribution. 
 $\trans$, $R$ and $\mu$ are unknown.
For notational convenience, we use $P_{s,a}$ and $P_{s,a,s'}$ to denote $P(\cdot|s,a)$ and $P(s'|s,a)$ respectively.

A policy $\pi$ chooses an action $a$ based on the current state $s \in \states$ and the time step $h \in [H]$. 
Formally, we define $\pi = \{\pi_h\}_{h = 1}^H$ where for each $h \in [H]$, $\pi_h : \states \to \actions$ maps a given state to an action.
The policy $\pi$ induces a (random) trajectory $\{s_1,a_1,r_1,s_2,a_2,r_2,\ldots,s_{H},a_{H},r_{H} \}$,
where $s_1 \sim \mu$, $a_1 = \pi_1(s_1)$, $r_1 \sim R(s_1,a_1)$, $s_2 \sim \trans(\cdot|s_1,a_1)$, $a_2 = \pi_2(s_2)$, etc.

The goal of RL is to find a policy $\pi$ that maximizes the expected total
reward, i.e.
$
\max_\pi \expect_{\pi} \left[\sum_{h=1}^{H} r_h \right] 
$
where the expectation is over the initial distribution state $\mu$, the transition operator $P$ and the reward distribution $R$.

As for scaling, we make the following assumption about the reward.
 As we will discuss in Section~\ref{sec:rel}, this is a more general assumption than the assumption made in most previous works.
 \begin{assumption}[Bounded Total Reward]\label{asmp:total_bound}
 	The reward satisfies that $r_h\geq 0$ for all $h\in [H]$. Besides, $\sum_{h=1}^H r_h\leq 1$ almost surely.
 \end{assumption}
\noindent Given a policy $\pi$, a level $h \in [H]$ and a state-action pair
$(s,a) \in \states \times \actions$, the $Q$-function is defined as:
\[
Q_h^\pi(s,a) = \expect_{\pi}\left[\sum_{h' = h}^{H}r_{h'}\mid s_h =s, a_h = a\right].
\]
Similarly, given a policy $\pi$, a level $h \in [H]$, the value function of a given state
$s \in \states$ is defined as: 
\[
V_h^\pi(s)=\expect_{\pi}\left[\sum_{h' = h}^{H}r_{h'}\mid s_h =s,
  \right].
\]
Then Bellman equation establishes the following identities for policy $\pi$ and $(s,a,h) \in \states \times \actions \times [H]$
 \begin{align*}
Q_h^\pi(s,a) =r_h(s,a)+ P_{s,a}^{\top}V_{h+1}^\pi~~~~ V_{h}^\pi(s) = \max_{a}Q_{h}^\pi(s,a).
 \end{align*}
 Throughout the paper, we let $V_{H+1}(s) = 0$ and $Q_{H+1}(s,a) = 0$ for notational simplicity.
We use $Q^*_h$ and $V^*_h$ to denote the optimal $Q$-function and $V$-function at level $h \in [H]$, which satisfies for any state-action pair $(s,a) \in \states \times \actions$, $Q^*_h(s,a) = \max_{\pi}Q^{\pi}_h(s,a)$ and $V^*_h(s) =\max_{\pi}V^{\pi}_h(s)$.

\paragraph{Reward-Free Reinforcement Learning}
Now we formally describe reward-free RL.
 Let $\epsilon,\delta \in (0,1)$ be the thresholds of sub-optimality and failure probability. Reward-free RL consists of two phases. In the exploration phase, the algorithm collects a dataset $\mathcal{D}$ by interacting with the environment without reward information, and in the planning phase, given any reward function $r$ satisfying Assumption \ref{asmp:total_bound}, the agent is asked to output an $\epsilon$-optimal policy with probability at least $1-\delta$. 
 
\paragraph{Dataset}
Formally, a dataset $\mathcal{D} = \{(s_h^k,a_h^k,s_{h+1}^k) \}_{(h,k)\in [H]\times [K]} $ consists of the trajectories of $K$ episodes. 
We also define $\{ N_{s,a,s'}(\mathcal{D})\}_{(s,a,s')\in \mathcal{S}\times \mathcal{A}\times\mathcal{S} } $ to be the visit count and \\ $\{ P_{s,a,s'}(\mathcal{D}) = \frac{N_{s,a,s'}(\mathcal{D})}{\sum_{\tilde{s}}N_{s,a,\tilde{s}}(\mathcal{D}) }\}_{(s,a,s')\in \mathcal{S}\times \mathcal{A} \times \mathcal{S}} $ be the empirical transition probability computed by $\mathcal{D}$, 
where $P_{s,a,s'}(\mathcal{D})$ is defined as $ \frac{1}{S}$ if $\sum_{\tilde{s}}N_{s,a,\tilde{s}}(\mathcal{D}) = 0$. We further define $N_{s,a}(\mathcal{D}) = \sum_{\tilde{s}}N_{s,a,\tilde{s}}(\mathcal{D})$ and $P_{s,a}(\mathcal{D})$ be the vector such that the value of its $s'$-th dimension is $P_{s,a,s'}(\mathcal{D})$ for all $(s,a)\in \mathcal{S}\times \mathcal{A}$. With the notation defined above, we let $N(\mathcal{D})$ and $P(\mathcal{D})$ be respectively the shorthands of $\{ N_{s,a}(\mathcal{D})  \}_{(s,a)\in \mathcal{S}\times \mathcal{A}}$ and $\{ P_{s,a}(\mathcal{D})  \}_{(s,a)\in \mathcal{S}\times \mathcal{A}}$.
 The performance of an algorithm is measured by how many episodes $K$ used in the exploration phase to make sure the planning phase succeeds.



\section{Technique Overview}\label{sec:tec}





The proposed algorithm has two main components: the sampling phase and the planning phase. In a high-level view,
 we first propose a sufficient condition (see Condition \ref{cond3}) for the agent to use the collect samples to learn an $\epsilon$-optimal policy for any reward function satisfying Assumption \ref{asmp:total_bound}. 
Then we apply a modified version of Rmax~\citep{brafman2002r} to obtain samples to satisfy Condition \ref{cond3} in the sampling phase. 

\subsection{Planning Phase}

\subsubsection{A Tight Sufficient Condition}
To obtain a near-optimal policy for any given reward function, a sufficient condition is to collect $\overline{N}$ samples for each $(s,a)$ pair, where $\overline{N}$ is some polynomial function of $S,A$ and $1/\epsilon$. However, some $(s,a)$ pairs might be rarely visited with any policy so it is hard to get enough samples for such pairs. To address this problem, we observe that such state-action pairs contribute little to the accumulative reward.
As mentioned in \cite{jin2020reward},
if the maximal expected visit count of $(s,a)$ is $\lambda(s,a)$, then $\overline{N}\lambda(s,a)$ samples of $(s,a)$ is sufficient for us to compute a good policy. 
Instead of considering each $(s,a)$ pair one by one, we hope to divide the state-action space into a group of disjoint subsets, such that the maximal expect visit count of each subset is proportionally to minimal visit count in this subset. This poses a sufficient condition for the dataset in the plan phase. 
\begin{condition}\label{cond2} 
    Let $K= \left\lfloor \log_{2}(2H/\epsilon) \right\rfloor$. Given the dataset $\mathcal{D}$, the state-action space $\mathcal{S}\times\mathcal{A}$ could be divided into $K+1$ subsets $\mathcal{S}\times \mathcal{A}= \mathcal{X}_1 \cup \mathcal{X}_2 \cup ...\cup \mathcal{X}_{K+1}$, such that, \\
    (1) For any $1\leq i \leq K$, $N_{s,a}(\mathcal{D})\geq N_i: = 4\frac{SH\iota}{2^i \epsilon^2}$ for any $(s,a)\in \mathcal{X}_i$;\\
    (2) For each $1\leq i \leq K+1$, it holds that $\sup_{\pi}\mathbb{E}_{\pi}\left[ \sum_{h=1}^H\mathbb{I}\left[(s_h,a_h)\in \mathcal{X}_{i} \right] \right] \leq \frac{H}{2^i}$.
\end{condition}

The following proposition shows this condition is sufficient.
The proof of Proposition \ref{pro1} is postpone to Appendix~\ref{app:omf}.
\begin{proposition}\label{pro1}
Suppose Condition \ref{cond2} holds for the dataset $\mathcal{D}$. 
Given any reward function $r$ satisfying Assumption \ref{asmp:total_bound}, with probability $1-4S^2A(\log_2(T_0H)+2)\delta$, \textsc{Q-Computing}$( P(\mathcal{D}),N(\mathcal{D}) ,r)$ (see Algorithm \ref{alg4}) returns an $\epsilon$-optimal policy.
\end{proposition}

In previous work on reward-free exploration \citep{jin2020reward,kaufmann2020adaptive,menard2020fast}, sufficient conditions similar to Condition \ref{cond2} have been proposed to prove efficient reward-free exploration. However, to obtain a dataset satisfying Condition \ref{cond2}, the sample complexity bound is at least polynomial in $H$ in the worst case, which is the main barrier of previous work. 
We give a simple counter-example to explain why Condition \ref{cond2} is hard to be satisfied without a $\poly\left(H\right)$ number of episodes. Suppose there is a state $\tilde{s}$, such that for any other $(s,a) \in \states \times \actions$ $P_{s,a,\tilde{s}}=\epsilon_{1}$, and for any action $a$ $P_{\tilde{s},a,\tilde{s}}=1$. Direct computation gives that $\lambda(s) = \Theta(H^2\epsilon_{1} )$. However, the probability that the agent never visit $\tilde{s}$ in $N$ episodes is at least $(1-H\epsilon_{1})^N\approx e^{-NH\epsilon_{1}} \approx e^{-\frac{N\lambda(s)}{H}}$. In the case $N\ll H$, the expected visit count in $N$ episodes is $N\lambda(s)$, while the empirical visit count could be $0$ with constant probability, which implies the expected visited number and the empirical visit can be very different in the $N=o(H)$ regime.

To address this problem, we observe that in the example above, the probability the agent reaches $\tilde{s}$ is relatively small. If we simply ignore $\tilde{s}$, the regret due to this ignorance is at most $O(H\epsilon_{1})  = O(\lambda(s)/H)$ instead of original regret bound of $O(\lambda(s))$.
This poses our main novel condition to plan for a near-optimal policy  given any reward function satisfying Assumption~\ref{asmp:total_bound}.
This is one of our key technical contributions.

\begin{condition}\label{cond3}
	Let $K =\left\lfloor\log_2(2H/\epsilon)\right\rfloor$.	The state-action space $\mathcal{S}\times\mathcal{A}$ could be divided into $K+1$  subsets $\mathcal{S}\times\mathcal{A}= \mathcal{X}_1\cup \mathcal{X}_2\cup...\cup\mathcal{X}_{K+1}$, such that, \\
	(1) $N(s,a)\geq N_i= 4\frac{  H(\iota+6S\ln(SAH/\epsilon))}{2^i\epsilon^2}$ for any $(s,a)\in\mathcal{X}_i$ for $1\leq i\leq K$;\\
	(2) Let $Z_i =\max\{ \min\{ \frac{H}{2^i\epsilon},H \},1 \}$ for each $1\leq i \leq K+1$.
	For each $1\leq i \leq K+1$, it holds that\\ $\sup_{\pi}\mathbb{P}_{\pi}[\sum_{h=1}^H \mathbb{I}\left[(s_h,a_h)\in \mathcal{X}_i \right] >Z_i ]\leq \epsilon$ and $\sup_{\pi}\mathbb{E}_{\pi}\left[   \min\{\sum_{h=1}^H \mathbb{I} \left[ (s_h,a_h)\in \mathcal{X}_i \right],Z_i \} \right]\leq  \frac{H}{2^i}$.
\end{condition}

Under Condition \ref{cond3}, the state-action space are divided into $K+1$ subsets according to their visit counts. 
For the state-action pairs with visit counts in $[N_i,N_{i-1})$, different with the second requirement in Condition \ref{cond2} we require that the  maximal \emph{truncated} expected visit count is strictly bounded proportionally to their visit counts. Let $\mathcal{E}_i$ be the set of trajectories satisfying that $\sum_{h=1}^H \mathbb{I}\left[ (s_h,a_h)\in \mathcal{X}_i \right]>Z_i$.
We also requires that the probability of $\mathcal{E}_i$  is no larger than $\epsilon$ for any policy. In fact, we directly pay loss of $\sup_{\pi}\mathbb{P}_{\pi}[\mathcal{E}_i ]$ due to ignoring $\mathcal{E}_i$ when computing the value function. On the other hand,  $Z_i$ is far less than $H$ when $i$ is relatively large, which enables us to  collect samples to satisfy Condition \ref{cond3}. 

The selection of $Z_i$ is quite tricky.  In one hand, we need $Z_i$ large enough so that it is possible to ensure $\sup_{\pi}\mathbb{P}_{\pi}\left[\sum_{h=1}^H  \mathbb{I}\left[ (s_h,a_h)\in \mathcal{X}_i \right] \right]$ no larger than $\epsilon$ (for example, by choosing $Z_i=H+1$, we can easily make this probability $0$), and in the other hand, we need $Z_i$ small enough to get rid of polynomial dependence on $H$. One possible solution is to set $Z_i$ to scale linear as the maximal expected visit count of $\mathcal{X}_i$, which plays a crucial role in the analysis.


\subsubsection{Planning using an Auxiliary MDP}

Suppose Condition \ref{cond3} holds for some dataset $\mathcal{D} $ with the partition $\{ \mathcal{X}_{i}\}_{i=1}^{K+1}$. 
Because we only require the truncated maximal expected visit is properly bounded in Condition \ref{cond3},  standard planning method (e.g., backward update in Algorithm \ref{alg4}) cannot work trivially.
The main difficulty here is that, to apply the bounds of $\sup_{\pi}\mathbb{E}_{\pi}\left[   \max\{\sum_{h=1}^H \mathbb{I} \left[ (s_h,a_h)\in \mathcal{X}_i \right],Z_i \} \right]$ and $\sup_{\pi}\mathbb{P}_{\pi}\left[ \mathcal{E}_i \right]$, we should set the reward $0$ if $\mathcal{X}_i$ has been visited for more than $Z_i$ times in an episode. A naive solution is to encode the visit counts of $\{\mathcal{X}_i \}_{i=1}^{K+1}$ into the state space. However, in this approach, the size of the new state space is exponential in $S$, which leads to exponential computational cost.
Due to the reason above, to our best of knowledge, no existing algorithms can direct learn such a \emph{truncated} MDP.

To address this problem, we consider an auxiliary MDP $\mathcal{M}^{\dagger} = \left \langle \mathcal{S}\cup s_{\mathrm{end}}, \mathcal{A}, r^{\dagger}, \hat{P}^{\dagger} ,\mu^{\dagger} \right \rangle$. 
Here $s_{\mathrm{end}}$ is an additional absorbing state. The reward function $r^{\dagger}$ is the same as $r$ except for an additional column $0$ for $s_{\mathrm{end}}$, and the transition probability $\hat{P}^{\dagger}$ is given by $P^{\dagger}_{s,a} = (1-\frac{1}{Z_i})\hat{P}_{s,a}+\frac{1}{Z_{i}}\textbf{1}_{s_{\mathrm{end}}}$ for any $(s,a)\in \mathcal{X}_i$ and $\hat{P}_{s_{\mathrm{end}},a} = \textbf{1}_{s_{\mathrm{end}}}$ for any $a$. 
In words, we add an absorbing state to the original MDP, such that the agent would fall into $s_{\mathrm{end}}$ if it visit $\mathcal{X}_i$ for $Z_i$ times in expectation for some $1\leq i \leq K+1$. Instead of learning the \emph{truncated} MDP, we consider a \emph{soft-truncated} MDP, which exponentially reduces computational cost. For more details, we refer the readers to Section \ref{sec:pf_plan}.

\subsection{Sampling Phase}
\label{sec:sampling}
Having identified the sufficient condition, we need to design an algorithm to collect a set of samples that satisfy this condition.

We make the partition  $\mathcal{S}\times \mathcal{A} =\cup_{i=1}^{K+1}\mathcal{X}_i$ by specifying $\mathcal{X}_i$ for $i=1,2,...,K+1$ one by one.  We divide the learning process into $K$ stages.
Take the first stage as an example. At the beginning of the first stage, we assign reward $1$ to all $(s,a)\in \mathcal{S}\times \mathcal{A}$, and proceeds to learn with this reward. Like \textsc{Rmax}, whenever the visit count of some $(s,a)$ pair is equal to or larger than $N_1$, we say this $(s,a)$ is $\emph{known}$ and set $r(s,a)=0$.  We will discuss the problem of regret minimization for this MDP with time-varying reward function later and simply assume the regret is properly bounded.
Defining $\mathcal{X}_1$ be the set of \emph{known} state-action pairs after the first stage, the statements in Condition \ref{cond3} holds trivially. Beside, the length of each stage is properly designed. Combining this with the bound of regret, we show that the maximal expected visit count of the \emph{unknown} state-action pairs is properly bounded. Because $\mathcal{X}_2 \subset (\mathcal{X}_1)^{C}$, we learn that the second part in the second statement in Condition \ref{cond3} holds for $\mathcal{X}_2$. We then continue to learn the second subset $\mathcal{X}_2$ and so on.

Note that in arguments above, we do not introduce $Z_i$ because $Z_i=H$ for the beginning stages by definition.  In the case $Z_i <H$,  there are two major problems.

\paragraph{The regret minimization algorithm } Classical regret minimization algorithms like \citep{azar2017minimax,zanette2019tighter} works in the regime $Z_i=H$, where no truncation occurs. However, in the case $Z_i \ll H$, the regret bounds by these algorithms depends on $H$ polynomially. To address this problem, we constructed an expanded MDP with truncated cumulative reward (see definition in Section \ref{sec:pf_sam} ), where the $Q$-function is strictly bounded by $Z_i$. In this way, we obtain desired regret bounds. We would like to mention that our algorithm is somewhat similar to recent work~\citep{zhang2020reinforcement} which addresses  the regret minimization problem with total-bounded reward function. More precisely, after re-scaling, the reward function in our regret minimization problem is also total bounded by $1$ and each single reward is bounded by $1/Z_i$. Although the reward function might vary in different episodes, we can provide efficient regret bounds in a similar way to the analysis in~\citep{zhang2020reinforcement}.

\paragraph{Bound of $\mathbb{P}\left[ \mathcal{E}_{i+1} \right]$} By the upper bound of regret (see Lemma \ref{lemma:bd_Ri}), we show that the maximal truncated expected visit count $\sup_{\pi}\mathbb{E}\left[ \min\{ \sum_{h=1}^H \mathbb{I}\left[ (s_h,a_h)\in (\mathcal{X}_1\cup...\cup \mathcal{X}_{i})^{C} , Z_i\right] \right]$ is properly bounded. Noting that $Z_{i+1}<Z_i$, we have that
\begin{align}
&\mathbb{P}\left[ \mathcal{E}_{i+1} \right] \leq \sup_{\pi}\mathbb{P}\left[ \sum_{h=1}^H \mathbb{I}\left[(s_h,a_h) \in (\mathcal{X}_1\cup...\cup \mathcal{X}_{i})^{C}   \right]  >Z_{i+1}   \right] \nonumber
\\ &\leq  \sup_{\pi}\mathbb{P}\left[ \sum_{h=1}^H \mathbb{I}\left[(s_h,a_h) \in (\mathcal{X}_1\cup...\cup \mathcal{X}_{i})^{C}   \right]  >Z_{i}   \right] +\frac{1}{Z_{i+1}}\sup_{\pi}\mathbb{E}\left[ \min\{ \sum_{h=1}^H \mathbb{I}\left[ (s_h,a_h)\in (\mathcal{X}_1\cup...\cup \mathcal{X}_{i})^{C} , Z_i\right] \right]\nonumber
\\ & \leq  \sup_{\pi}\mathbb{P}\left[ \sum_{h=1}^H \mathbb{I}\left[(s_h,a_h) \in (\mathcal{X}_1\cup...\cup \mathcal{X}_{i-1})^{C}   \right]  >Z_{i}   \right] +\frac{1}{Z_{i+1}}\sup_{\pi}\mathbb{E}\left[ \min\{ \sum_{h=1}^H \mathbb{I}\left[ (s_h,a_h)\in (\mathcal{X}_1\cup...\cup \mathcal{X}_{i})^{C} , Z_i\right] \right].\label{eq_tec_sam_1}
\end{align}
By properly choosing the value of $Z_i$, we show that the second term in \textbf{RHS} of \eqref{eq_tec_sam_1} could be bounded by $O(\epsilon)$. Then by induction, we show that $\mathbb{P}\left[ \mathcal{E}_{i+1} \right]\leq K\epsilon$. Noting that $K = \left\lfloor \log_{2}(2H/\epsilon) \right\rfloor$ is a logarithmic term, we can bound the probability of $\mathbb{P}\left[ \mathcal{E}_{i+1} \right]$ properly.


Following the arguments above, we set the number of episodes in each stage  to be $T_0 :=C_1\frac{SA(\iota+S)l}{\epsilon^2}$ where $C_1$ is some large enough constant and $l$ is a  poly-logarithmic term in $(S,A,H,1/\epsilon)$. At the beginning of an episode in the $i$-th stage, we assign reward $1$ to a state-action pair if its visit number is less than $N_i$ and otherwise $0$. We then apply Algorithm \ref{alg2} to minimize regret in each stage, and finally obtain  $\mathcal{X}_1,\mathcal{X}_2,...,\mathcal{X}_{K+1}$.
 
 For more technical details, we refer the reader to Section \ref{sec:pf_sam} and \ref{sec:pf_plan}.

\section{Algorithm and Proofs}\label{sec:pro}
\begin{algorithm}[t]
\caption{\textsc{Main Algorithm: Staged Sampling + Truncated Planning}\label{alg:main} }
\begin{algorithmic}[1]
	\State $(\mathcal{D}, \{ \mathcal{X}_i \}_{i=1}^{K+1})\leftarrow$ \textsc{Staged Sampling} (Algorithm~\ref{alg1});
	\State Given any reward function $r$ satisfying Assumption~\ref{asmp:total_bound}, return $\pi\leftarrow$\textsc{Truncated Planning}$(\mathcal{D}, \{ \mathcal{X}_i \}_{i=1}^{K+1},r)$ (Algorithm~\ref{alg3}).
\end{algorithmic}
\end{algorithm}

Similar as in Section \ref{sec:tec}, our proof consists of two parts, one for the sampling phase and another for the planning phase. We propose the main lemmas for these two parts respectively.
\begin{lemma}\label{lemma:sample}
By running Algorithm \ref{alg1}, with probability $1-K\left( 2(\log_2(T_0 H) +1)\log_{2}(T_0H) + 4S^2A(\log_2(H)+2) \right)\delta$, we can collect a dataset $\mathcal{D}$ and obtain the partition $\{X_{i}\}_{i=1}^{K+1}$ such that Condition \ref{cond3} holds for the collected dataset $ \mathcal{D}$. Besides, we consumes at most $KT_0 = \tilde{O}(\frac{SA(\iota+S)}{\epsilon^2})$ episodes to run Algorithm~\ref{alg1}.
\end{lemma}

\begin{lemma}\label{lemma:plan}
Assuming Condition \ref{cond3} holds for the collected dataset $ \mathcal{D}$ with partition $\{\mathcal{X}_i\}_{i=1}^{K+1}$ , with probability $1-4S^2A(\log_2(T_0H)+2)\delta$, Algorithm \ref{alg3} can compute an $\epsilon$-optimal policy using these samples for any reward function $r$ satisfying Assumption \ref{asmp:total_bound}.
\end{lemma}

Theorem \ref{thm1} follows by combining Lemma \ref{lemma:sample} with Lemma \ref{lemma:plan} and replacing $\delta$ by $\mathrm{poly}(S,A,1/\epsilon,\log(H))\delta$. The rest part of this section is devoted to the proofs of Lemma \ref{lemma:sample} and Lemma \ref{lemma:plan}.

\subsection{Sampling Phase: Proof of Lemma \ref{lemma:sample}}\label{sec:pf_sam}

\begin{algorithm}[t]
	\caption{\textsc{Staged Sampling} \label{alg1}
		}
	\begin{algorithmic}[1]
		\Statex \textbf{Initialize:  $\mathcal{D}\leftarrow \emptyset$, $\mathcal{Y}_1 \leftarrow \mathcal{S}\times \mathcal{A}$}  
		\For{$i=1,2,...,K$}
		\State $(D_{i} ,\mathcal{Y}_{i+1}) \leftarrow$ \textsc{TRVRL}($i$, $\mathcal{Y}_i$);
		\State $\mathcal{D}\leftarrow  \mathcal{D}\cup \mathcal{D}_i$;
		\State $\mathcal{X}_i \leftarrow \mathcal{Y}_i /\mathcal{Y}_{i+1}$
		\EndFor
		\Statex $\mathcal{X}_{K+1}\leftarrow \mathcal{Y}_{K+1}$;
        \State Return $(\mathcal{D}, \{\mathcal{X}_i\}_{i=1}^{K+1})$.
	\end{algorithmic}
\end{algorithm}

\begin{algorithm}
	\caption{\textsc{\textbf{T}runcated \textbf{R}eward-\textbf{V}arying \textbf{R}einforcement \textbf{L}earning (TRVRL)}  \label{alg2}
	}
	\begin{algorithmic}[1]
	\Statex \textbf{Input:} The stage index $i$, the unknown set $\mathcal{Y}$, $\epsilon_1 =\min\{\frac{\iota}{T_0H},\frac{\iota^2}{T_0^2H^3} \}$, $\iota_1 =\iota+ S\ln(1/\epsilon_1) $;
		\Statex \textbf{Initialize:} Trigger set $\mathcal{L} \leftarrow \{ 2^{j-1}| 2^{j}\leq T_0H, j=1,2,\ldots \}$; $N_i \leftarrow 4\frac{S\iota H}{2^i \epsilon^2}$; $\mathcal{D}\leftarrow \emptyset$.
		\For{$(s,a,s',h)\in \mathcal{S}\times  \mathcal{A}\times\mathcal{S}$}
		\State $N(s,a)\leftarrow 0$; $n(s,a)\leftarrow 0$; \State $N(s,a,s')\leftarrow 0$; $\hat{P}_{s,a,s'}\leftarrow 0$.
		\EndFor 
		\For {$k=1,2,..., T_0$}
		\For {$h=1,2,...,H$}
		\State Observe $(s_{h}^k,z_h^k)$;
		\State Take action $ a_h^k= \arg\max_{a}Q_h(s_h^k,z_h^k,a)$; \label{line:choose_action}
		\State Receive reward $r_h^k$ and observe $(s_{h+1}^k, z_{h+1}^k )$.
		\State Set $(s,a,s',r)\leftarrow (s_h^k,a_h^k,s_{h+1}^k,r_h^k)$;
		\State Set $N(s,a) \leftarrow  N( s,a )+1$, \,$N(s,a,s') \leftarrow   N(s,a,s')+1$.
		\Statex \verb|\\| \emph{Update empirical transition probability}
		\If {$N(s,a)\in \mathcal{L}$}  \label{line:rp_update_start}
		\State Set $\hat{P}_{s,a,\tilde{s}} \leftarrow  N(s,a,\tilde{s}) /N(s,a)$ for all $\tilde{s} \in \mathcal{S}$.
		\State Set $n(s,a)\leftarrow N(s,a)$;
		\State Set TRIGGERED = TRUE.
		\EndIf \label{line:rp_update_end}
		\EndFor
		\State $\mathcal{D}\leftarrow \mathcal{D}\cup \{(s_1^k,a_1^k,s_2^k,...,s_{H}^k,a_H^k,s_{H+1}^k)\}$;
		\Statex \verb|\\| \emph{Update $Q$-function when the probability is updated or the unknown set changes}
		\State $\mathcal{Y}^{k+1} \leftarrow \{(s,a)| N(s,a)\leq N_i,(s,a)\in \mathcal{Y} \}$
		\If {TRIGGERED or $\mathcal{Y}^{k+1} \neq \mathcal{Y}^{k}$}
		\For{$(s,a)\in \mathcal{S}\times\mathcal{A}\times\mathcal{S}$}
		\If{$(s,a)\in \mathcal{Y}^{k+1}$}
		\State Set $\hat{P}_{s,z,a,\tilde{s},z'} \leftarrow \mathbb{I}[z'=z+1]\cdot \hat{P}_{s,a,\tilde{s}}$ for $1\leq z\leq Z_i$ and  all $\tilde{s}\in \mathcal{S}$;
		\State Set $\hat{P}_{s,Z_i+1,a,\tilde{s},Z_i+1}\leftarrow \hat{P}_{s,a,\tilde{s}}$ for  all $\tilde{s}\in \mathcal{S}$;
		\Else  
		\State Set $\hat{P}_{s,z,a,\tilde{s},z'} \leftarrow \mathbb{I}[z'=z]\cdot \hat{P}_{s,a,\tilde{s}}$ for $1\leq z\leq Z_i+1$ and  all $\tilde{s}\in \mathcal{S}$;
		\EndIf
		\EndFor
		\For{$(s,z,a)\in \mathcal{S}\times[Z_i+1] \times \mathcal{A}$}
		\State Set $V_{H+1}(s,z,a)\leftarrow 0$;
		\EndFor
		\For{$h=H,H-1,...,1$}
		\For{$(s,z,a)\in \mathcal{S}\times[Z_i+1] \times \mathcal{A}$}
		\State 			\vspace{-0.5cm} 	
		Set
		\begin{align} 
		~~~& r^{k+1}(s,z,a) = \mathbb{I}\left[(s,a)\in \mathcal{Y}^{k+1}\cap z\leq Z_i \right],\nonumber
		\\&b_h(s,z,a)\leftarrow  \sqrt{\frac{ 4  \mathbb{ V}(\hat{P}_{s,z,a} ,V_{h+1}) \iota_1  }{ \max\{n(s,a),1 \} }}+ \frac{14Z_i\iota_1}{3 \max\{n(s,a) ,1\}  }+3\epsilon_1,  \label{equpdate1}  
		\\ 	\hspace{-20ex} 	& Q_h(s,z,a)\leftarrow \min\{    r^{k+1}(s,z,a)+\hat{P}_{s,z,a} V_{h+1} +b_h(s,z,a)    ,Z_i\}, \label{equpdate2}
		\\& V_h(s,z) \leftarrow \max_{a}Q_h(s,z,a).\nonumber
		\end{align}
		\vspace{-3ex}
		\EndFor
		\EndFor
		\State Set TRIGGERED = FALSE
		\EndIf
		
		\EndFor
		\State \textbf{Return} $(\mathcal{D},\mathcal{Y}^{T_0+1})$.
	\end{algorithmic}
\end{algorithm}

As mentioned in Section \ref{sec:tec}, we aim to collect samples such that Condition \ref{cond3} holds. Our algorithm proceeds in $K+1$ stages, where each stage consists of $T_0$ episodes. Therefore, at most $KT_0=\tilde{O}(\frac{SA(\iota+S)}{\epsilon^2})$ episodes are needed to run Algorithm~\ref{alg1}.
In an episode, saying the $k$-th episode in the $i$-th stage, we define $\mathcal{Y}_k = \{(s,a)| N^k(s,a)< 2N_{i} \}$ to be the set of unknown state-action pairs. In particular, we define $\mathcal{Y}_{i}: = \mathcal{Y}^{\overline{k}(i)}$ where $\overline{k}(i)$ is the first episode in the $i$-th stage.

To learn the \emph{unknown} state-action pairs, we adopt the idea of Rmax by setting reward function to be $r(s,a)  = \mathbb{I}\left[(s,a)\in \mathcal{Y}_k \right]$.
However, by the definition of Condition \ref{cond3}, it suffices to assign reward $1$ to the first $Z_i$ visits to $\mathcal{Y}_i$.  So it corresponds to learn a policy to maximize 
\begin{align}
    \mathbb{E}_{\pi}\left[ \min\{\sum_{h=1}^H  \mathbb{I} \left[(s_h,a_h)\in \mathcal{Y}^k \right],Z_i\} \right]. \nonumber
\end{align}
To address this learning problem, we consider an expanded MDP $\mathcal{M}^k = \left\langle \mathcal{S}^k,\mathcal{A}^k,  P^k ,r^k, \mu^k \right\rangle$, where 
\begin{align}
& \mathcal{S}^k  = \mathcal{S}\times [Z_i+1];\nonumber \\
& \mathcal{A}^k = \mathcal{A} ; \nonumber\\
&r^k(s,z,a) = \mathbb{I}[(s,a)\in \mathcal{Y}^k, z\leq Z_i],\quad  \forall (s,z,a)\in \mathcal{S}\times [Z_i+1] \times\mathcal{A} ; \nonumber\\
& P^k(s',z'|s,z,a) = P(s'|s,a)\mathbb{I}\left[z'=z+1\cap (s,a)\in \mathcal{Y}^k \right]+ P(s'|s,a)\mathbb{I}\left[z'=z\cap (s,a)\notin \mathcal{Y}^k \right],\nonumber
\\& \quad \quad \quad \quad \forall (s,a)\in \mathcal{S}\times \mathcal{A}, 1\leq z\leq Z_i; \nonumber \\
&P^k(s',Z_i+1|s,Z_i+1,a) = P(s'|s,a), \quad  \forall (s,a)\in \mathcal{S}\times\mathcal{A}; \nonumber\\
& \mu^k(s,z) = \mu(s)\mathbb{I}\left[z=1 \right].\nonumber
\end{align}

Roughly speaking, a state in $\mathcal{M}^k$ not only represent its position in $\mathcal{S}$, but also records the reward the agent has collected in current episode. We then define the pseudo regret in the $i$-th stage as:
\begin{align}
 R_{i}:= \sum_{k \text{ in stage } i} (\sup_{\pi}\mathbb{E}_{\pi}[\sum_{h=1}^H r_h^k]  -\sum_{h=1}^H r_h^k ),\nonumber
\end{align}
where $r_h^k$ is a shorthand of $r^k(s_h^k,z_h^k,a_h^k)$. We show that $R_{i}$ could be bounded properly in a similar way to \cite{zhang2020reinforcement}.
\begin{lemma}\label{lemma:bd_Ri}
For any $1\leq i\leq K$, by running Algorithm \ref{alg2} with input $i$, with  probability $1-\left( 2(\log_2(T_0 H) +1)\log_{2}(T_0H) + 4SA(\log_2(Z_i)+2) \right)\delta$,  $R_i$ is bounded by  
\begin{align}
& O\left(Z_i \sqrt{SAl_i( \iota + S\ln(T_0^3H^4/\iota^3)) T_0}+ Z_il_i ( \iota + S\ln(T_0^3 H^4/\iota^3)) SA\right)
\\& = O\left(Z_i \sqrt{SAl_i( \iota + S\ln(SAH/\epsilon)) T_0}+ Z_il_i ( \iota + S\ln(SAH/\epsilon)) SA\right) 
\end{align}
where $l_i = \log_2(Z_i)$.
\end{lemma}
The proof of Lemma \ref{lemma:bd_Ri} is postponed to Appendix.\ref{sec:mpf_lemma_bd_Ri} due to limitation of space.

Let $i$ be fixed. Recall that $\overline{k}_i$ is the first episode in the $i$-th stage. We define 
$u_{k} = \sup_{\pi} \mathbb{E}_{\pi}\left[ \sum_{h=1}^H r_h^{k} \right]$, $\overline{u}^i = u_{\overline{k}(i)}$ and $\underline{u}^{i} = u_{\underline{k}(i)}$ where $\underline{k}(i)$ is the index of the last episode in the $i$-th stage. Because $r^k$ is non-increasing in $k$, $\mu^k$ is also non-increasing in $k$.
 If $\underline{u}^{i}>\frac{H}{2^i}$, then by Lemma \ref{lemma:bd_Ri} and the definition of $T_0$ we have that there exists a constant $C_2$ such that
 \begin{align}
      \sum_{k \text{ in stage } i}\sum_{h=1}^H r_h^k \geq T_0\underline{u}_i - C_2(Z_i \sqrt{SAl_i( \iota + S\ln(SAH/\epsilon)) T_0}+ Z_il_i ( \iota + S\ln(SAH/\epsilon)) SA).\nonumber
 \end{align}
 By choosing $C_ 1\geq 8C_2$, we have that
\begin{align}
    \sum_{k \text{ in stage } i}\sum_{h=1}^H r_h^k & \geq T_0\underline{u}_i - C_2(Z_i \sqrt{SAl_i( \iota + S\ln(SAH/\epsilon)) T_0}+ Z_il_i ( \iota + S\ln(SAH/\epsilon)) SA\iota)
\\ & \geq \frac{C_1}{2}\frac{SAH(\iota+6S\ln(SAH/\epsilon))}{2^i \epsilon^2} 
 \\&   >\frac{C_1}{8}SAN_i.\nonumber
\end{align}
By choosing $C_1\geq 16$, we learn that $\sum_{k \text{ in stage } i}\sum_{h=1}^H r_h^k> 2SAN_i$. On the other hand, each $(s,a)$ could provide at most $N_i$ rewards in the $i$-th stage, which implies that $\sum_{k \text{ in stage } i}\sum_{h=1}^H r_h^k\leq 2SAN_i$. This leads to a contradiction. We then have that $\underline{u}_i \leq \frac{H}{2^i}$.

Again because the reward function is non-increasing in $k$, we have that 
\begin{align} \overline{u}_{i+1} \leq \underline{u}_{i}\leq \frac{H}{2^i} . \label{eq3} \end{align}

Define $p_{i} = \sup_{\pi} \mathbb{P}_{\pi} \left[   \sum_{h=1}^H \mathbb{I}\left[ (s_h,a_h)\in \mathcal{Y}_i \right]> Z_i    \right]$. Then we have that
\begin{align}
p_{i+1} &\leq  \sup_{\pi} \mathbb{P}_{\pi} \left[  Z_{i+1}< \sum_{h=1}^H \mathbb{I}\left[ (s_h,a_h)\in \mathcal{Y}_{i+1} \right]\leq Z_i   \right] +\sup_{\pi} \mathbb{P}_{\pi} \left[   \sum_{h=1}^H \mathbb{I}\left[ (s_h,a_h)\in \mathcal{Y}_{i+1} \right]> Z_{i}    \right]\nonumber
\\ & \leq \mathbb{I}\left[2^{i+1}\epsilon \geq 1\right]\frac{ \underline{u}_i}{Z_{i+1}} +p_{i}\nonumber
\\ & \leq \epsilon+p_{i} .\nonumber
\end{align}
By induction, we can obtain that $p_{i}\leq i\epsilon\leq (K+1)\epsilon$ and $\sum_{i=1}^{K}p_{i}\leq (K+1)^2\epsilon$.   
We claim that Condition \ref{cond3} holds by defining $\mathcal{X}_i = \mathcal{Y}_i/\mathcal{Y}_{i+1}$ for $1\leq i \leq K$ and $\mathcal{X}_{K+1} = \mathcal{Y}_{K+1}$. \\
(1)  By the definition of $\mathcal{Y}_{i+1}$ , we learn that for any $(s,a)\in \mathcal{X}_{i}$, $N(s,a)\geq 2N_{i+1}\geq N_i$.\\
(2) By the arguments above, we have that for each $1\leq i\leq K+1$.
\[\sup_{\pi}\mathbb{P}_{\pi}\left[\sum_{h=1}^H \mathbb{I}\left[(s_h,a_h)\in \mathcal{X}_i >Z_i \right]  \right] \leq \sup_{\pi}\mathbb{P}_{\pi}\left[\sum_{h=1}^H \mathbb{I}\left[(s_h,a_h)\in \mathcal{Y}_i >Z_i \right]  \right]=p_i \leq (K+1)\epsilon \]
and
\[ \sup_{\pi}\mathbb{E}_{\pi}\left[ \max\{ \sum_{h=1}^H \mathbb{I}\left[(s_h,a_h)\in \mathcal{X}_i \right],Z_i \}   \right]\leq \overline{u}_i \leq \underline{u}_{i-1}\leq \frac{H}{2^{i-1}}.  \]
Noting that there are exactly $K$ stages and each stage consists of $T_0 = C_1\frac{SA(\iota+6S\ln(SAH/\epsilon))\log_2(H)}{\epsilon^2}$ episodes, we prove that we can collect a dataset satisfying Condition \ref{cond3} within 
$$C_1\frac{SAK(\iota+6S\ln(SAH/\epsilon))\log_2(H)}{\epsilon^2} = C_1\frac{SA(\iota+6S\ln(SAH/\epsilon))\ln(H/\epsilon)\ln(H)}{\epsilon^2} = \tilde{O}(\frac{SA(S+\ln(1/\delta))}{\epsilon^2}) $$ episodes.

\subsection{Planning Phase: Proof of Lemma \ref{lemma:plan}}\label{sec:pf_plan}

\begin{algorithm}[t]
	\caption{Truncated Planning\label{alg3}	}
	\begin{algorithmic}[1]
		\Statex \textbf{Input:} The partition $\{\mathcal{X}_i\}_{i=1}^{K+1}$; the dataset $\mathcal{D}$; the reward function $r$.
		\State  \textbf{Initialize:} $r(s_{\mathrm{end}},a)\leftarrow 0$; $P(\cdot|s_{\mathrm{end}},a)\leftarrow \textbf{1}_{s_{\mathrm{end}} }$ for all $a\in \mathcal{A}$;\\
		\hspace{4ex}\quad \quad \quad  $\hat{P}\leftarrow \{ P_{s,a}(\mathcal{D}) \}_{(s,a)\in \mathcal{S}\times \mathcal{A}}$; $N\leftarrow \{ N_{s,a}(\mathcal{D})\}_{(s,a)\in \mathcal{S}\times \mathcal{A}}$.
		\For{$i=1,2,...,K+1$}
		\State $\hat{P}^{\dagger}_{s,a}\leftarrow (1-\frac{1}{Z_{i}})\hat{P}_{s,a} + \frac{1}{Z_i}\textbf{1}_{s_{\mathrm{end}}}$ for any $(s,a)\in \mathcal{X}_i$;
		\EndFor
	\State $Q\leftarrow \textsc{Q-Computing}(\hat{P}^{\dagger},N,r)$ \State $\pi_h(s)\leftarrow \arg\max_{a}Q_h(s,a)$, $\forall s,h $;
 	\State \textbf{Return} $\pi$.
	\end{algorithmic}
\end{algorithm}

\begin{algorithm}[t]
	\caption{Q-Computing\label{alg4}	}
	\begin{algorithmic}[1]
		\Statex \textbf{Input: $P$, $N$, $r$, $\epsilon_1 =\min\{\frac{\iota}{T_0H},\frac{\iota^2}{T_0^2H^3} \}$, $\iota_1 =\iota+ S\ln(1/\epsilon_1) $;}
		\For{$(s,a,h)\in \mathcal{S}\times\mathcal{A}\times[H]$}
		\State $Q_h(s,a)\leftarrow 1$;
		\EndFor
		\For{$(a,h)\in \mathcal{A}\times\mathcal{H}$}
		\State{ $Q_h(s_{\mathrm{end}},a)\leftarrow 0$;}
		\EndFor
		\For{$h=H,H-1,...,1$}
		\For{$(s,a)\in \mathcal{S}\times \mathcal{A}$}
		\State 		\vspace{-0.5cm}  
		\begin{align}
 \hspace{-20ex}&	b_h(s,a)\leftarrow 2\sqrt{\frac{ \mathbb{V}(\hat{P}_{s,a},V_{h+1} )\iota_1}{N(s,a)}}+ \frac{14\iota_1}{3N(s,a)}; \label{eq:alg4_update1}
	\\\hspace{-20ex}&	   Q_h(s,a)\leftarrow \min\{ r(s,a)+ \hat{P}_{s,a}V_{h+1}+b_h(s,a),1 \}+3\epsilon_1; \label{eq:alg4_update2}
		   \end{align}
		   		  \vspace{-3ex}
		\EndFor
		\EndFor
	\State \textbf{Return:} $\{Q_h(s,a) \}_{(s,a,h)\in \mathcal{S}\times \mathcal{A}\times [H]}$.
	\end{algorithmic}
\end{algorithm}

 Suppose we have a dataset $\mathcal{D}$ satisfying Condition \ref{cond3} with partition $\{\mathcal{X}_i\}_{i=1}^{K+1}$. Let $\hat{P}_{s,a}$ and $N(s,a)$ be the shorthand of $P_{s,a}(\mathcal{D})$ and $N_(s,a)(\mathcal{D})$ respectively.
 Denote the empirical transition and visit count of $(s,a)$ as $\hat{P}_{s,a}$ and $N(s,a)$ respectively.
 
Let $\epsilon_1 \min\{\frac{\iota}{T_0H} ,\frac{\iota^2}{T_0^2H^3} \} $ , $\delta_1 = \delta \epsilon_1^{S}$ and $\iota_1 = \ln(1/\delta_1) \leq \iota+ S(\ln(T_0H/\iota)+\ln(T_0^2H^3/\iota^2)) = \iota + S\ln(T_0^3H^4/\iota^3) $.
Let $\mathcal{L} = \{[i_1\epsilon_1,i_2\epsilon_1,...,i_S \epsilon_1]^{T} | i_1,i_2,...,i_{S}\in \mathbb{Z}  \} \cap [0,1]^{S}$. 
 Define $\mathcal{G}$ to be the event where
\begin{align}
 & |(P_{s,a}-\hat{P}_{s,a})x|\leq \sqrt{\frac{2\mathbb{V}(P_{s,a}, x) \iota_1 }{N_i}} +\frac{\iota_1}{3N_i},\quad \forall x\in \mathcal{L}, \forall (s,a)\in \mathcal{X}_i, i=1,2,...,K ;\label{eq_goodevent1}
    \\ & |(P_{s,a}-\hat{P}_{s,a})x|\leq \sqrt{\frac{4\mathbb{V}(\hat{P}_{s,a}, x) \iota_1 }{N_i}} +\frac{14\iota_1}{3N_i},\quad \forall x\in \mathcal{L},  \forall (s,a)\in \mathcal{X}_i, i=1,2,...,K . \label{eq_goodevent2} 
\end{align}
hold.
 By Bernstein's inequality and empirical Bernstein inequality (Lemma \ref{empirical bernstein}), via a union bound, we have that $\mathbb{P}\left[ \mathcal{G} \right]\geq 1-4S^2A(\log_2(T_0H)+2)|\mathcal{L}|\delta_1 \geq 1-4S^2A(\log_2(T_0H)+2)\delta$. We will continue to prove conditioned on $\mathcal{G}$. We first establish a concentration bound based on $\mathcal{G}$.
\begin{lemma}\label{lemma:aux0}
    Conditioned on $\mathcal{G}$, for  any $v\in [0,1]^{S}$, any $i\in [K]$ and any $(s,a)\in \mathcal{X}_i$, it holds that
    \begin{align}
        |(P_{s,a}-\hat{P}_{s,a})v|\leq \min\{ \sqrt{\frac{2\mathbb{V}(P_{s,a},v)\iota_1}{N_i}}+\frac{\iota_1}{3N_i}+3\epsilon_1 ,  \quad \sqrt{\frac{4\mathbb{V}(\hat{P}_{s,a},v)\iota_1}{N_i}}+\frac{14\iota_1}{3N_i}+3\epsilon_1  \}.
    \end{align}
\end{lemma}
 \begin{proof}
 The conclusion follows easily by using \eqref{eq_goodevent1} and \eqref{eq_goodevent2} with the projection of $v$ on $\mathcal{L}$.
 \end{proof}
 
As mentioned in Section \ref{sec:tec}, we consider the reward-free auxiliary MDP $\mathcal{M}^{\dagger}=\left \langle \mathcal{S}\cup \{s_{\mathrm{end}}\}, \mathcal{A}, P^{\dagger},\mu\right\rangle$.  
  The transition function $P^{\dagger}_{s,a} = (1-\frac{1}{Z_i})P_{s,a}+\frac{1}{Z_{i}}\textbf{1}_{s_{\mathrm{end}}}$ for  all $(s,a)\in \mathcal{X}_i$ and $P^{\dagger}_{s_{\mathrm{end}},a} =\textbf{1}_{s_{\mathrm{end}} }$ for any $a$. 
  We first show that for any policy, the value function of $\mathcal{M}^{\dagger}$ is  $\widetilde{O}(\epsilon)$-closed to that of $\mathcal{M}$. 
  
 \begin{lemma}\label{lemma:approx}
 	For any policy $\pi$ and reward function $r$ satisfying Assumption \ref{asmp:total_bound} and $r_{s_{\mathrm{end}}}=0$, define $V^{\pi}_{1}$ and $V^{\dagger\pi}_{1}$ be the value function under $\mathcal{M}$ and $\mathcal{M}^{\dagger}$ with $\pi$ respectively. We then have 
 	\[V^{\dagger\pi}_{1}\leq V^{\pi}_{1}\leq V^{\dagger\pi}_{1}+ 4(K+1)^2\epsilon.\]
 \end{lemma}
\begin{proof}
For notational convenience, we use $\mathbb{E}_{\pi,\mathcal{M}}$ ($\mathbb{P}_{\pi,\mathcal{M}}$) and $\mathcal{E}_{\pi,\mathcal{M}^{\dagger}}$ ($\mathbb{P}_{\pi,\mathcal{M}^{\dagger}}$) to denote the expectation (probability) under $\mathcal{M}$ and $\mathcal{M}^{\dagger}$ following $\pi$ respectively.

The left side is obvious because the reward is always positive. Now we prove for the right side. Recall that $\mathcal{E}_i$ is the set of trajectories satisfying that $\sum_{h=1}^H \mathbb{I}\left[ (s_h,a_h)\in \mathcal{X}_i \right]>Z_i$.
Define $\mathcal{E} = \cup_{i=1}^{K+1} \mathcal{E}_i$ and $\overline{\mathcal{U}}$ be the set of trajectories satisfying that $ \{s_{H+1} = s_{\mathrm{end}} $. Define $\mathcal{U} = \overline{\mathcal{U}}\cap \mathcal{E}^{C}$
We claim that $\mathbb{P}_{\pi,\mathcal{M}^{\dagger} }\left[\mathcal{U}\right]\leq C(K+1)\epsilon$.  To prove this claim, we define $\mathcal{U}_{i}$ be the set of trajectories satisfying that the state before $s_{\mathrm{end}}$ is in $\mathcal{X}_i$. 
Because $\mathcal{U}_{i}$ only depends on the next states of first $Z_{i}$ visits to $\mathcal{X}_{i}$, we define $\tau_i$ be the first time among the $Z_i$ visits that  the next state of $\mathcal{X}_i$ is $s_{\mathrm{end}}$ ($\tau_i=Z_{i}+1$ if no such event occurs), and $\gamma_i$ be the number of visits to $\mathcal{X}_i$ in this episode conditioned on the event $\tau_i = Z_i+1$. So $\gamma_i$ is independent of $\tau_i$.
\begin{align}
\mathbb{P}_{\pi, \mathcal{M}^{\dagger}}\left[ \mathcal{U}_i\right] &  = \mathbb{E}_{\pi,\mathcal{M}^{\dagger}}\left[   \mathbb{I}\left[   \tau_i \leq \gamma_i  \right]    \right]\nonumber
\\ & = \sum_{j=1}^{Z_{i}}\mathbb{P}_{\pi,\mathcal{M}^{\dagger}}\left[ \gamma_i = j \right]\mathbb{E}_{\pi,\mathcal{M}^{\dagger}}\left[\tau_i\leq j \right]\nonumber
\\ & = \sum_{j=1}^{Z_{i}}\frac{j}{Z_{i}} \mathbb{P}_{\pi,\mathcal{M}^{\dagger}}\left[ \gamma_i = j \right] \nonumber
\\ & \leq \frac{1}{Z_i} \mathbb{E}_{\pi,\mathcal{M}}\left[ \max\{\sum_{h=1}^{H}\mathbb{I}\left[ (s_h,a_h)\in \mathcal{X}_i \right] ,Z_i \} \right] \nonumber
\\ & \leq 2\epsilon.\nonumber
\end{align}
Define $\mathcal{V}$ be the set of  trajectories satisfying $s_{H+1}\neq s_{\mathrm{end}}$. For a trajectory  $\Gamma = (s_1,a_1,s_2,a_2,...,s_H,a_H,s_{H+1})$ in $\mathcal{V}$, we define $r(\Gamma) = \sum_{h=1}^H r(s_h,a_h)$.  By definition, we have that
\begin{align}
V_{1}^{\pi} & = \sum_{\Gamma \in \mathcal{V}}r(\Gamma)\mathbb{P}_{\pi, \mathcal{M}}(\Gamma) \nonumber
\\  & \leq \sum_{\Gamma \in \mathcal{E}^C\cap \mathcal{V}}r(\Gamma)\mathbb{P}_{\pi,\mathcal{M}}(\Gamma) +2(K+1)^2\epsilon\nonumber
\\ & \leq V_{1}^{\dagger\pi} + \sum_{\Gamma \in \mathcal{E}^C\cap \mathcal{V}}r(\Gamma) (\mathbb{P}_{\pi,\mathcal{M}}(\Gamma)- \mathbb{P}_{\pi,\mathcal{M}^{\dagger}}(\Gamma) ) +2(K+1)^2\epsilon \nonumber
\\ & \leq V_{1}^{\dagger\pi}  + \sum_{\Gamma \in \mathcal{E}^C\cap \mathcal{V}} (\mathbb{P}_{\pi ,\mathcal{M}}(\Gamma)- \mathbb{P}_{\pi, \mathcal{M}^{\dagger}}(\Gamma) )  +2(K+1)^2\epsilon \nonumber 
\\ & \leq V_{1}^{\dagger\pi}  +\mathbb{P}_{\pi,\mathcal{M}^{\dagger}}\left[ \mathcal{U}\right]+2(K+1)^2\epsilon \label{eq5}
\\ & \leq V_{1}^{\dagger \pi}+4(K+1)^2\epsilon.\label{eq6}
\end{align}
Here, Inequality \eqref{eq5} is by the fact that $\sum_{\Gamma \in \mathcal{E}^C\cap \mathcal{V}} \mathbb{P}_{\pi,\mathcal{M}^{\dagger}}  = 1-\mathbb{P}_{\pi,\mathcal{M}^{\dagger}}\left[U \right] -\mathbb{P}_{\pi,\mathcal{M}^{\dagger}}\left[ \mathcal{E}\right]\geq 1- \mathbb{P}_{\pi,\mathcal{M}^{\dagger}}\left[U \right] - \mathbb{P}_{\pi,\mathcal{M} }\left[ \mathcal{E}\right] = \sum_{\Gamma \in \mathcal{E}^{C}\cap \mathcal{V} }\mathbb{P}_{\pi,\mathcal{M}}(\Gamma)- \mathbb{P}_{\pi,\mathcal{M}^{\dagger}}\left[\mathcal{U} \right] $. The proof is completed.
\end{proof}

Instead of learning $\mathcal{M}$, we aim to learn $\mathcal{M}^{\dagger}$.
Let $\hat{P}_{s,a}$ be the empirical transition computed by the collected samples. As described in Algorithm \ref{alg3}, for each $1\leq i\leq K+1$, we define $\hat{P}^{\dagger}_{s,a} = (1-\frac{1}{Z_i})\hat{P}_{s,a} +\frac{1}{Z_{i}}\textbf{1}_{s_{\mathrm{end}}}$ for $(s,a)$ in $\mathcal{X}_i$.   We update backward the $Q$-function and value function for the MDP $\tilde{M}$ as below. 
\begin{align}
& V_{H+1}  = 0\nonumber
\\ & b_{h}(s,a) = 2\sqrt{\frac{ \mathbb{V}(\hat{P}^{\dagger}_{s,a}, V_{h+1})\iota_1 }{N(s,a)}} +14 \frac{\iota_1}{N(s,a)}+3\epsilon_1 ,\quad \forall (s,a);\nonumber
\\ & Q_h(s,a) = \min\{   r_h(s,a)+ \hat{P}^{\dagger}_{s,a}V_{h+1} +b_h(s,a)    ,1\} ,\quad \forall (s,a);\label{eq:update}
\\ & V_h(s) = \max_{a}Q_h(s,a),\quad  \forall s.\nonumber
\end{align}
The final output policy $\pi$ is induced by the $Q$-function above. We first verify the $Q$-function is optimistic, i.e.,
\begin{lemma}\label{lemma:Qopt}
	Conditioned on $\mathcal{G}$, $Q_h(s,a)\geq Q^{\dagger*}_h(s,a)$ for any $(s,a,h)$.
\end{lemma}
We postpone the proof of Lemma \ref{lemma:Qopt} to Appendix.\ref{sec:mpf_lemma_Qopt} due to limitation of space.

Without loss of generality, we assume $\mu = \textbf{1}_{s_1}$.
Now we bound the gap $V_{1}^*(s_1)-V^{\pi}_1(s_1)$. Assuming $\mathcal{G}$ holds, for any $(s,a)\in \mathcal{S}\times\mathcal{A}$, we have
\begin{align}
Q_{h}(s,a)-r_h(s,a)-P^{\dagger}_{s,a}V_{h+1}& \leq b_h(s,a)+ (\hat{P}^{\dagger}_{s,a}-P_{s,a}^{\dagger})V_{h+1} \nonumber
\\& = b_h(s,a) +(1-\frac{1}{Z_{i(s,a)}})\sum_{s'}\left(\hat{P}_{s,a,s'}-P_{s,a,s'}\right)\cdot\left(V_{h+1}(s')- P^{\dagger}_{s,a}V_{h+1} \right) \nonumber
\\ & \leq b_h(s,a) + \sqrt{\frac{ 2\mathbb{V}(P^{\dagger},V_{h+1})\iota_1 }{N(s,a)}}+\frac{14\iota_1}{3N(s,a)}+3\epsilon_1\nonumber
\\ & =  3\sqrt{\frac{ \mathbb{V}(\hat{P}^{\dagger}_{s,a},V_{h+1} )\iota_1 }{N(s,a)}}   +\sqrt{ \frac{2 \mathbb{V}(P_{s,a}^{\dagger}, V_{h+1} )\iota_1}{N(s,a)}}+         \frac{5\iota_1}{N(s,a)}+6\epsilon_1.\label{eq:local9}
\end{align}

Define $u$ be the vector such that $u_{s} = (V_{h+1}(s') -P^{\dagger}_{s,a}V_{h+1} )^2$. By Lemma~\ref{lemma:aux0}, we have that
\begin{align}
   \mathbb{V}(\hat{P}^{\dagger}_{s,a}, V_{h+1}) & \leq  \hat{P}^{\dagger}_{s,a}u^2\nonumber
   \\ & \leq P^{\dagger}_{s,a}u^2+ \sqrt{\frac{\mathbb{V}(P^{\dagger}_{s,a},u^2 ) \iota_1 }{N(s,a)}}+ \frac{\iota_1}{3N(s,a)}+3\epsilon_1
  \\ & \leq \frac{3}{2}P^{\dagger}_{s,a}u^2 +\frac{4\iota_1}{3N(s,a)}+ 3\epsilon_{1} 
  \\& = \frac{3}{2}\mathbb{V}(P^{\dagger}_{s,a},V_{h+1})+\frac{4\iota_1}{3N(s,a)}+ 3\epsilon_{1} .\label{eq:local10}
\end{align}


Define $\beta_h(s,a): =\min\{6 \sqrt{ \frac{ \mathbb{V}(P_{s,a}^{\dagger}, V_{h+1} )\iota_1}{N(s,a)}}  + \frac{9\iota_1}{N(s,a)}+9\epsilon_1, 1 \}.$ Combining \eqref{eq:local9}, \eqref{eq:local10} and the trivial bound $Q_h(s,a)\leq 1$, we have that
\begin{align}
    Q_{h}(s,a)-r_h(s,a)-P^{\dagger}_{s,a}V_{h+1}& \leq \beta_h(s,a).\label{eq:glo4}
\end{align}
Then we have
\begin{align}
V_{1}^{\dagger*}(s_1)-V^{\dagger\pi}_1(s_1) & \leq  V_1(s_1)-V^{\dagger\pi}_1(s_1)\nonumber
\\ & \leq \beta_1(s_1,a_1) +P^{\dagger}_{s_1,a_1}(V_{2}-V_{2}^{\dagger\pi} ) \nonumber
\\ & \leq ...
\\ & \leq \sum_{s,a,h}w_h(s,a,\pi) \beta_h(s,a),\nonumber
\end{align}
where $w_h(s,a,\pi): = \mathbb{E}_{\pi,\mathcal{M}^{\dagger}}\left[ \mathbb{I}\left[ (s_h,a_h) =(s,a)\right]  \right]$ .
Define $\omega_{i}^{\dagger}(\pi) = \sum_{(s,a)\in \mathcal{X}_i}\sum_h w_h(s,a,\pi)$ for $1\leq i \leq K+1$. Define $t_{Z_i}$ be $Z_i$-th time such that $(s_h,a_h)\in \mathcal{X}_i$ if $\sum_{h=1}^H \mathbb{I}\left[(s_h,a_h)\in \mathcal{X}_i \right]\geq Z_i$ and otherwise $H+1$.
We have
\begin{align}
&\omega_{i}^{\dagger}(\pi)  = \mathbb{E}_{\pi,\mathcal{M}^{\dagger}} \left[\sum_{h=1}^H \mathbb{I}\left[(s_h,a_h)\in \mathcal{X}_i \right]  \right] \nonumber
\\ & = \mathbb{E}_{\pi, \mathcal{M}^{\dagger}} \left[\min \{\sum_{h=1}^H \mathbb{I}\left[(s_h,a_h)\in \mathcal{X}_i \right],Z_{i}  \} \right]  +    \mathbb{E}_{\pi, \mathcal{M}^{\dagger}} \left[ \mathbb{I}\left[\sum_{h=1}^H \mathbb{I}\left[(s_h,a_h)\in \mathcal{X}_i \right]  \geq Z_i  \right]\cdot \left(\sum_{h=1}^H \mathbb{I}\left[(s_h,a_h)\in \mathcal{X}_i \right] -Z_i\right)\right]\nonumber\nonumber
\\ & \leq  \mathbb{E}_{\pi, \mathcal{M}} \left[\min \{\sum_{h=1}^H \mathbb{I}\left[(s_h,a_h)\in \mathcal{X}_i \right],Z_{i}  \}\right]  +   \mathbb{E}_{\pi, \mathcal{M}^{\dagger}} \left[ \mathbb{I}\left[\sum_{h=1}^H \mathbb{I}\left[(s_h,a_h)\in \mathcal{X}_i \right]  \geq Z_i  \right]\cdot \left(\sum_{h=1}^H \mathbb{I}\left[(s_h,a_h)\in \mathcal{X}_i \right] -Z_i\right)\right]\nonumber
\\ & \leq O\left( \frac{H}{2^i}\right) +\sum_{h'=1}^H \mathbb{P}(t_{Z_i}=h')\mathbb{E}_{\pi,\mathcal{M}^{\dagger}}\left[ \sum_{h=h'+1}^H \mathbb{I}\left[(s_h,a_h)\in \mathcal{X}_i \right]  \Big|t_{Z_i}=h' \right]\nonumber
\\ & \leq  O\left( \frac{H}{2^i}\right)  +\sum_{h'=1}^H \mathbb{P}(t_{Z_i}=h')\frac{1}{Z_i}\nonumber
\\ & \leq  O\left(\frac{H}{2^i}+Z_iK\epsilon \right)\nonumber
\\ & \leq O\left(   \frac{HK}{2^i}   \right).\nonumber
\end{align}
 
Combining this with the fact that $\iota+6S\ln(SAH/\epsilon)\geq  \iota +S(4\ln(H)+3\ln(T_0/\iota))) =\iota +S\ln(T_0^3H^4/\iota^3) \iota_1$ (assuming $S,A,H\geq 10$), we have
\begin{align}
&\sum_{s,a,h} w_h(s,a,\pi)\beta_h(s,a) \nonumber \\& \leq  O\left( \sum_{(s,a)\notin \mathcal{X}_{K+1}}\sum_{h}w_h(s,a,\pi)\sqrt{\frac{ \mathbb{V}(P^{\dagger}_{s,a},V_{h+1} ) \iota_1}{ N(s,a) }}+ \sum_{(s,a)\notin \mathcal{X}_{K+1}}\sum_h w_h(s,a,\pi)\frac{\iota_1}{N(s,a)} \right) + w^{\dagger}_{K+1}(\pi)\nonumber
\\ & \leq \sum_{i=1}^K O\left(  \sum_{(s,a)\in \mathcal{X}_i }\sum_h w_h(s,a,\pi)\sqrt{\frac{ \mathbb{V}(P^{\dagger}_{s,a},V_{h+1} )\iota_1 }{ N_i}}  +   \sum_{(s,a)\in \mathcal{X}_i }\sum_h w_h(s,a,\pi)\frac{\iota_1}{N_i} \right) + O(K\epsilon)\nonumber
\\ & \leq \sum_{i=1}^K O\left(           \sqrt{\frac{w^{\dagger}_i(\pi)}{N_i}}\cdot \sqrt{ \sum_{h=1}^H w_h(s,a,\pi) \mathbb{V}(P^{\dagger}_{s,a}, V_{h+1})\iota_1 }  + \frac{w_{i}^{\dagger}(\pi)\iota_1}{N_i}   \right) +O(K\epsilon). \nonumber
\end{align}
Note that
\begin{align}
& \sum_{h=1}^H w_h(s,a,\pi) \mathbb{V}(P^{\dagger}_{s,a},V_{h+1})  \nonumber
\\ &  = \mathbb{E}_{\pi,\mathcal{M}^{\dagger}}\left[   \sum_{h=1}^H\left( P^{\dagger}_{s,a}(V_{h+1})^2 - (  P^{\dagger}_{s,a} V_{h+1}   )^2   \right)  \right] \nonumber
\\ &  \leq \mathbb{E}_{\pi,\mathcal{M}^{\dagger} } \left[ \sum_{h=1}^{H} (V_{h}(s_{h})  )^2  - (  P^{\dagger}_{s,a} V_{h+1}   )^2    \right] \nonumber
\\ & \leq 2\mathbb{E}_{\pi,\mathcal{M}^{\dagger}}\left[   \sum_{h=1}^H \left(r(s_h,a_h) +\beta_h(s_h,a_h) \right)\right]\nonumber 
\\ & \leq  2 + 2\sum_{s,a,h}w_{h}(s,a,\pi)\beta_h(s,a) .\nonumber
\end{align}
We then have
\begin{align}
\sum_{s,a,h}w_h(s,a,\pi)\beta_h(s,a) &\leq O\left( M\epsilon \sqrt{2+2\sum_{s,a,h}w_h(s,a,\pi)\beta_h(s,a)} +M^2\epsilon^2  +M\epsilon \right) .\nonumber
\end{align}
By solving the inequality $x\leq O(K\epsilon \sqrt{2+2x}+K^2\epsilon^2)$, we learn that 
\begin{align}
V^{\dagger*}_1(s_1)-V^{\dagger \pi}_1(s_1)\leq \sum_{s,a,h}w_h(s,a,\pi)\beta_h(s,a)\leq O\left(K\epsilon + K^2\epsilon^2  \right) .\nonumber
\end{align}
Recall that by Lemma \ref{lemma:approx}, we have $|V_{1}^{\pi}(s_1)-V^{\dagger \pi}_1(s_1)|\leq O\left((K+1)^2\epsilon \right)$ and $|V^{\dagger*}_1(s_1) -V^{*}_1(s_1)|\leq O\left((K+1)^2\epsilon \right)$. We then finally conclude that
\begin{align}
    V^{*}_1(s_1)-V^{ \pi}_1(s_1)\leq \sum_{s,a,h}w_h(s,a,\pi)\beta_h(s,a)\leq O\left(K^2\epsilon \right)   .\nonumber
\end{align}
By rescaling $\epsilon$, the proof is completed.

\bibliography{ref}
\bibliographystyle{plainnat}

\newpage
\appendix
\paragraph{Organization} In Appendix~\ref{app:tec_lemma}, we present some basic technical lemmas. In Appendix~\ref{sec:mpf_sam} and \ref{sec:mpf_plan} we give the missing proofs in Section~\ref{sec:pro}. The left missing proofs are presented in Appendix~\ref{app:omf}.

\section{Technical Lemmas}\label{app:tec_lemma}
\begin{lemma}[Bennet's Inequality]\label{bennet}
Let $Z,Z_1,...,Z_n$  be i.i.d. random variables with values in $[0,1]$ and let $\delta>0$. Define $\mathbb{V}Z = \mathbb{E}\left[(Z-\mathbb{E}Z)^2 \right]$. Then we have
\begin{align}
\mathbb{P}\left[ \left|\mathbb{E}\left[Z\right]-\frac{1}{n}\sum_{i=1}^n Z_i  \right| > \sqrt{\frac{  2\mathbb{V}Z \ln(2/\delta)}{n}} +\frac{\ln(2/\delta)}{n} \right]]\leq \delta.\nonumber
\end{align}
\end{lemma}

\begin{lemma}[Theorem 4 in  \cite{maurer2009empirical}  ]\label{empirical bernstein}
Let $Z,Z_1,...,Z_n$ ($n\geq 2$) be i.i.d. random variables with values in $[0,1]$ and let $\delta>0$. Define $\bar{Z} = \frac{1}{n}\sum_{i=1}^n Z_{i}$ and $\hat{V}_n  = \frac{1}{n}\sum_{i=1}^n (Z_i- \bar{Z})^2$. Then we have
\begin{align}
\mathbb{P}\left[ \left|\mathbb{E}\left[Z\right]-\frac{1}{n}\sum_{i=1}^n Z_i  \right| > \sqrt{\frac{  2\hat{V}_n \ln(2/\delta)}{n-1}} +\frac{7\ln(2/\delta)}{3(n-1)} \right] \leq \delta.\nonumber
\end{align}
\end{lemma}

\begin{lemma}[Lemma 10 in \cite{zhang2020model}]\label{self-norm}
Let $(M_n)_{n\geq 0}$ be a martingale such that $M_0=0$ and $|M_n-M_{n-1}|\leq c$ for some $c>0$ and any $n\geq 1$. Let $\mathrm{Var}_{n} = \sum_{k=1}^n \mathbb{E}\left[  (M_{k}-M_{k-1})^2 |\mathcal{F}_{k-1}\right]$ for $n\geq 0$, where $\mathcal{F}_k = \sigma(M_1,...,M_{k})$. Then for any positive integer $n$, and any $\epsilon,\delta>0$, we have that
\begin{align}
\mathbb{P} \left[       |M_n|\geq 2\sqrt{2}\sqrt{\mathrm{Var}_n \ln(1/\delta)} +2\sqrt{\epsilon \ln(1/\delta)} +2c\ln(1/\delta) \right]\leq 2(\log_2(\frac{nc^2}{\epsilon}) +1)\delta.\nonumber
\end{align}
\end{lemma}

\begin{lemma}[Lemma 11 in \cite{zhang2020reinforcement}]\label{lemma2}
Let $\lambda_1,\lambda_2,\lambda_4\geq 0$, $\lambda_3\geq 1$ and $i' =\log_2(\lambda_1)$.	Let $a_{1},a_{2},...,a_{i'}$ be non-negative reals such that $a_{i}\leq \lambda_{1}$ and $a_{i}\leq \lambda_{2}\sqrt{a_{i+1}+ 2^{i+1} \lambda_3 } +\lambda_4$ for any $1\leq i\leq i'$. 
	 Then we have that $a_{1}\leq \max\{ (\lambda_2 +\sqrt{\lambda_2^2+\lambda_4} )^2  ,\lambda_{2}\sqrt{8\lambda_3}  +\lambda_4 \}$   .
\end{lemma}

\section{Miss Proofs in Section \ref{sec:pf_sam}}\label{sec:mpf_sam}
\subsection{Proof of Lemma \ref{lemma:bd_Ri}}\label{sec:mpf_lemma_bd_Ri}

\textbf{Lemma \ref{lemma:bd_Ri} (Restate)}\emph{For any $1\leq i\leq K$, by running Algorithm \ref{alg2} with input $i$, with  probability $1-2(\log_2(T_0H)+1)\log_2(T_0 H)\delta -4SA(\log_2(Z_i)+2)H\delta$,  $R_{i}$ is bounded by  
\begin{align}
& O\left(Z_i \sqrt{SAl_i( \iota + S\ln(T_0^3H^4/\iota^3)) T_0}+ Z_il_i ( \iota + S\ln(T_0^3 H^4/\iota^3)) SA\right)
\\& = O\left(Z_i \sqrt{SAl_i( \iota + S\ln(SAH/\epsilon)) T_0}+ Z_il_i ( \iota + S\ln(SAH/\epsilon)) SA\right) \label{eq:add2},
\end{align}
where  $l_i = \log_2(Z_i)$.} Here \eqref{eq:add2} holds because $T_0$ is bounded by $\mathrm{poly}(S,A,\ln(H),1/\epsilon)\iota$.

To facilitate the proof, we introduce some additional notations. We use $V_h^k$, $Q_h^k$, $r^k$ and $P^{k,\ddagger}$ to denote respectively the value function, $Q$-function, reward function and transition probability of the extended MDP for the $k$-th episode. Besides, we use
$b^k_h(s,z,a)$, $\hat{P}^{k,\ddagger}_{s,z,a}$ and $n^k(s,a)$ to denote respectively the values of $b_h(s,z,a)$, $\hat{P}^{k,\ddagger}_{s,z,a}$ and $\max\{n(s,a),1 \}$ in \eqref{equpdate1} when computing $Q_h^k(s,z,a)$. We also use $\hat{P}^k_{s,a}$ to denote the empirical transition probability of $(s,a)$ in \eqref{equpdate1} when computing $Q_h^k(s,z,a)$.
 Recall that $P_{s,a}$   denotes  the true transition probability of $(s,a)$ under the original MDP

Let $i$ be fixed.
We first introduce the good event in the $i$-th stage. 

Recall $\epsilon_1 \min\{\frac{\iota}{T_0H} ,\frac{\iota^2}{T_0^2H^3} \} $ , $\delta_1 = \delta \epsilon_1^{S}$ and $\iota_1 = \ln(1/\delta_1) \leq \iota S(\ln(T_0H/\iota)+\ln(T_0^2H^3/\iota^2))$.
Also recall that $\mathcal{L} = \{[i_1\epsilon_1,i_2\epsilon_1,...,i_S \epsilon_1]^{T} | i_1,i_2,...,i_{S}\in \mathbb{Z}  \} \cap [0,Z_i]^{S}$. 
For $0\leq j \leq \log_2(Z_i)+1$,we define $\mathcal{G}_{i}^{(j)}(s,a)$ be the event where
\begin{align}
    & |(\hat{P}^{(j)}_{s,a}-P_{s,a})x| \leq \sqrt{\frac{2\mathbb{V}(P_{s,a},x)\iota_1}{2^j}}+ \frac{Z_i\iota_1}{3\cdot 2^j} ,\quad \forall x\in \mathcal{L} ;   \label{eq_gsa_1}
  \\   & |(\hat{P}^{(j)}_{s,a}-P_{s,a})x| \leq \sqrt{\frac{4\mathbb{V}(\hat{P}^{(j)}_{s,a},x)\iota_1}{2^j}}+ \frac{14Z_i\iota_1}{3\cdot 2^j} ,\quad \forall x\in \mathcal{L}     .  \label{eq_gsa_2}
\end{align}
 hold, where $\hat{P}^{(j)}_{s,a,s'}$ is the empirical transition probability computed by the first $2^j$ samples in the current stage. By Bernstein's inequality and empirical Bernstein inequality, we have that $\mathbb{P}\left[ \mathcal{G}_{i,j}(s,a) \right]\geq 1-4|\mathcal{L}|\delta_1 = 1-4\delta$. Then the good event $\mathcal{G}_i$ is defined as $\cup_{s,a,j} \mathcal{G}_{i}^{(j)}(s,a)$. Via a union bound over all possible $(s,a,j)$, we have that $\mathbb{P}\left[ \mathcal{G}_i \right]\geq 1-4S^2A(\log_2(Z_i)+2)\delta$. In the rest of this section, we will prove conditioned on $\mathcal{G}_i$.
 
 We first establish a concentration bound for $(P_{s,a}-\hat{P}^{(j)}_{s,a})v$ for $v\in [0,Z_i]^S$.
 Define $\mathrm{Proj}_{\mathcal{L}}(v) = \arg\min_{v'\in \mathcal{L}}\|v-v'\|_{1}$.
Conditioned on $\mathcal{G}_i$ holds, we have that for any $(s,a)$ and $v\in [0,Z_i]^S$,
\begin{align}
   | (P_{s,a}-\hat{P}^{(j)}_{s,a})v|\leq & \sqrt{\frac{2\iota_1 \mathbb{V}(P_{s,a},\mathrm{Proj}_{\mathcal{L}}(v) )}{2^j}} +\epsilon_{1}+ \frac{Z_i\iota_1}{3\cdot 2^j}\nonumber
  \\& \leq  \sqrt{\frac{2\iota_1 \mathbb{V}(P_{s,a},v )}{2^j}} +3\epsilon_{1} +\frac{Z_i\iota_1}{3\cdot 2^j}. \label{eq:rb_1}
\end{align}
Similarly, we have
\begin{align}
     | (P_{s,a}-\hat{P}^{(j)}_{s,a})v|\leq  \sqrt{\frac{4\iota_1 \mathbb{V}(\hat{P}_{s,a},v )}{2^j}} +3\epsilon_{1} +\frac{14Z_i\iota_1}{3\cdot 2^j} .\label{eq:rb_2}
\end{align}

For $1\leq k \leq T_0$, we 
let $r_h^k$ be shorthand of $r^k(s_h^k,z_h^k,a_h^k)$.
We define the optimal $Q$-function  for the extended MDP as
\begin{align}
  &  Q_h^{*k}(s,z,a)  = \sup_{\pi} \mathbb{E}_{\pi}\left[\sum_{h'=h}^{H}r^k_{h'}| (s_h^k,z_h^k,a_h^k) = (s,z,a) \right], \quad \forall  (s,z,a,h);\nonumber
\\  & V_h^{*k}(s,z) = \max_{a}Q_h^{*k}(s,z,a),\quad  \forall (s,z,a). \nonumber
\end{align}
By the definition of $r^k$, it is obvious that $Q_h^{*k}(s,z,a)\leq Z_i$ for any $(s,z,a,h)$.
Recall that
\begin{align}
    R_{i} = \sum_{k \text{ in stage} i}\left(\sup_{\pi}\mathbb{E}_{\pi}\left[\sum_{h=1}^H r_h^k \right] - \sum_{h=1}^H r_h^k \right).\nonumber
\end{align}
So it corresponds to bounding for Algorithm \ref{alg2} that
\begin{align}
    \text{Regret}:=  \sum_{k= 1}^{T_0}\left( V^{*k}_{1}(s_1,1) - \sum_{h=1}^H r_h^k \right).\nonumber
\end{align}

Define the policy $\pi^k$ by  $\pi^k_h(s,z) = \arg\max_{a} Q^k_h(s,z,a)$.
It is not hard to verify $Q^k$ is optimistic because the size of support of $\mathbb{P}(\cdot|s,z,a)$ is at most $S$ for any $(s,z,a)$. 
\begin{lemma}\label{lemma:alg2_Qopt}
Conditioned on $\mathcal{G}_i$, $Q^k_h(s,z,a)\geq Q^{*k}_h(s,z,a)$ for any $(s,z,a)\in \mathcal{S}\times [Z_i+1] \times \mathcal{A}$ and any $1\leq k\leq T_0$.
\end{lemma}

\begin{proof}
We will prove by backward induction. Assuming $Q^k_{h'}(s,z,a)\geq Q^{*k}_{h'}(s,z,a)$ for any $(s,z,a)$ and $h<h'\leq H+1$.
By the update rule \eqref{equpdate1} and \eqref{equpdate2}, we have that
\begin{align}
    Q_h^k(s,z,a)&  = \min \{ r^{k+1}(s,z,a)+ \hat{P}^{k,\ddagger}_{s,z,a}V_{h+1}^k + b_h^k(s,z,a)  , Z_i \} \nonumber
    \\ & = \min \{ r^{k+1}(s,z,a)+ P^{k,\ddagger}_{s,z,a}V_{h+1}^{k} +   \zeta_h^k(s,z,a) +b_h^k(s,z,a) ,Z_i \} \nonumber
    \\ & \geq \min \{r^{k+1}(s,z,a)+  P^{k,\ddagger}_{s,z,a}V_{h+1}^{*k} + \zeta_h^k(s,z,a) +b_h^k(s,z,a),Z_i\} \nonumber
    \\& \geq \min\{Q^{*k}_h(s,z,a)+ \zeta_h^k(s,z,a) +b_h^k(s,z,a),Z_i \},\nonumber
\end{align}
where $\zeta_h^k(s,z,a) =vV_{h+1}^{k}$. So it suffices to verify conditioned on $\mathcal{G}_{i}$, it holds that $\zeta_h^k(s,z,a) +b_h^k(s,z,a)\geq 0$.

\begin{lemma}\label{lemma:aux1}
Conditioned on $\mathcal{G}_i$, for any $v\in [0,Z_i]^{S\times Z_i}$ and any proper $(k,s,z,a)$, it holds that
\begin{align}
    &| (\hat{P}^{k,\ddagger}_{s,z,a} -P^{k,\ddagger}_{s,z,a})v|
    \\ &\leq \min \left\{ \sqrt{\frac{2\mathbb{V}(P^{k,\ddagger}_{s,z,a},v)\iota_1 }{n^k(s,a)}} +3\epsilon_{1}+ \frac{Z_i\iota_1}{3n^k(s,a)},\quad \sqrt{\frac{4 \mathbb{V}(\hat{P}^{k,\ddagger}_{s,z,a},v) \iota_1 }{n^k(s,a)}}+3\epsilon_1+ \frac{14Z_i\iota_1}{3n^k(s,a)} \right\}.
\end{align}
\end{lemma}

\begin{proof}

Direct computation gives that
\begin{align}
 &  |(\hat{P}^{k,\ddagger}_{s,z,a} -P^{k,\ddagger}_{s,z,a})v |  \nonumber
 \\  & =\Big| \sum_{s'}\sum_{z'}  \left(\hat{P}^k_{s,a,s'} - P_{s,a,s'}\right)\cdot\mathbb{I}\left[P^{k,\ddagger}_{s,z,a,s',z'}>0 \right]\cdot (v_{s,z} -P^{k,\ddagger}_{s,z,a}v) \Big|\nonumber
 \\ & =\Big| \sum_{s'} \left(\hat{P}^k_{s,a,s'} - P_{s,a,s'}\right) \sum_{z'} \mathbb{I}\left[P^{k,\ddagger}_{s,z,a,s',z'}>0 \right]\cdot (v_{s,z} -P^{k,\ddagger}_{s,z,a}v) \Big| .\nonumber
\end{align}
Let $x$ be the $S$-dimensional vector such that $x_s =\sum_{z'} \mathbb{I}\left[P^{k,\ddagger}_{s,z,a,s',z'}>0 \right]\cdot v_{s',z'} $ and $x' = x -\min_{s}x_{s}\textbf{1}$. Because for any $(s,z,a,s')$, there is at most one $z'$ such that $P^{k,\ddagger}_{s,z,a,s',z'}>0$, we have that $x_{s}^2  =\sum_{z'} \mathbb{I}\left[P^{k,\ddagger}_{s,z,a,s',z'}>0 \right]\cdot v^2_{s',z'} $ and $x'\in [0,Z_i]^{S}$. Then by \eqref{eq:rb_1}, \eqref{eq:rb_2} and the definition of $n^k(s,a)$, we have that 
\begin{align}
   & | (\hat{P}^{k,\ddagger}_{s,z,a} -P^{k,\ddagger}_{s,z,a})v| 
    \\ & \leq \min\left\{\sqrt{\frac{2\iota_1 \mathbb{V}(P^k_{s,a},x' )}{n^k(s,a)}} +3\epsilon_1 +\frac{Z_i\iota_1}{3n^k(s,a)},  \quad  \sqrt{\frac{4\iota_1 \mathbb{V}(\hat{P}^k_{s,a},x' )}{n^k(s,a)}} +3\epsilon_1 +\frac{14Z_i\iota_1}{3n^k(s,a)} \right\}.\label{eq:add1}
\end{align}

Noting that
\begin{align}
   & \mathbb{V}(\hat{P}^{k,\ddagger}_{s,z,a},v) 
    \\ &= \sum_{s'}\hat{P}^k_{s,a,s'} \sum_{z'} \mathbb{I}\left[P^{k,\ddagger}_{s,z,a,s',z'}>0 \right]v_{s',z'}^2   - \big(\sum_{s'} \hat{P}^k_{s,a,s'}\sum_{z'}\mathbb{I}\left[P^{k,\ddagger}_{s,z,a,s',z'}>0 \right]v_{s',z'}\big)^2
    \\ & = \sum_{s'}\hat{P}^{k}_{s,a,s'}x_{s'}^2  - (\hat{P}_{s,a}^k x)^2
    \\ & = \mathbb{V}(\hat{P}^k_{s,a},x)
    \\ & = \mathbb{V}(\hat{P}^k_{s,a},x').
\end{align}
By \eqref{eq:add1}, we have that
\begin{align}
   & | (\hat{P}^{k,\ddagger}_{s,z,a} -P^{k,\ddagger}_{s,z,a})v|
   \\ & \leq \min \left\{  \sqrt{\frac{2\iota_1 \mathbb{V}(P^{k,\ddagger}_{s,z,a},v)}{n^k(s,a)}} +3\epsilon_1 +\frac{Z_i\iota_1}{3n^k(s,a)},\quad \sqrt{\frac{4\iota_1 \mathbb{V}(\hat{P}^{k,\ddagger}_{s,z,a},v)}{n^k(s,a)}} +3\epsilon_1 +\frac{14Z_i\iota_1}{3n^k(s,a)}. \right\}.
\end{align}
The proof is completed.
\end{proof}

By Lemma~\ref{lemma:aux1}, we have that
\begin{align}
    |\zeta_h^k(s,z,a) | \leq \sqrt{\frac{4\mathbb{V}(\hat{P}^{k,\ddagger}_{s,z,a},V_{h+1}^{k} )\iota_1}{n^k(s,a)}} +3\epsilon_1 +\frac{14Z_i\iota_1}{3n^k(s,a)}.\label{eq:lemma1_2}
\end{align}
Combining this with the definition of $b_h^k(s,z,a)$, we conclude that $Q_h^k(s,z,a)\geq Q^{*k}_{h}(s,z,a)$.  We finish the proof by noting that $Q^k_{H+1}(s,z,a) = Q^{*k}_{H+1}(s,z,a) = 0$ for any $(s,z,a)$.
\end{proof}

By \eqref{eq:rb_1} and similar arguments as above, we can bound for any $(s,z,a)$
\begin{align}
Q_h^k(s,z,a)& \leq r^k(s,z,a) + \hat{P}^{k,\ddagger}_{s,z,a} V^k_{h+1} + b_h^k(s,z,a)\nonumber
\\ & \leq r^k(s,z,a) +P^{k,\ddagger}_{s,z,a} V^k_{h+1}  +b_h^k(s,z,a) +\sqrt{\frac{4\mathbb{V}(P^{k,\ddagger}_{s,z,a},V_{h+1}^{k}) \iota_1}{n^k(s,a)}} + 3\epsilon_1 +\frac{Z_i\iota_1}{3n^k(s,a)}.\label{eq:local1}
\end{align}


Now we aim to bound $\mathbb{V}(\hat{P}^{k,\ddagger}_{s,z,a},V_{h+1}^k$ by $\mathbb{V}(P^{\ddagger,k}_{s,z,a},V_{h+1}^k)$. Because $\mathbb{V}(p,x) = \min_{\lambda\in \mathbb{R}}p(x-\lambda\textbf{1})^2$, we can find $\lambda\in \mathbb{R}$, such that $\mathbb{V}(P_{s,z,a}^{k,\ddagger}, V_{h+1}^k) = P_{s,z,a}^{k,\ddagger}(V_{h+1}^k -\lambda\textbf{1})^2$. Let $v = V_{h+1}^k -\lambda\textbf{1}$, then we have that $\mathbb{V}(\hat{P}^{k,\ddagger}_{s,z,a},V_{h+1}^k)\leq \hat{P}^{k,\ddagger}_{s,z,a}v^2$.

Again by Lemma~\ref{lemma:aux1}, we have that
\begin{align}
    \hat{P}^{k,\ddagger}_{s,z,a}v^2 -P_{s,z,a}^{k,\ddagger}v^2 &\leq Z_{i}\left(\sqrt{\frac{2 \mathbb{V}(P^{k,\ddagger}_{s,z,a}, \frac{v^2}{Z_i} ) \iota_1}{n^k(s,a) }} +3\epsilon_1+ \frac{Z_i\iota_1}{3n^k(s,a)} \right)
    \\ & \leq  Z_{i}\left(\sqrt{\frac{2 P^{k,\ddagger}_{s,z,a}v^2  \iota_1}{n^k(s,a) }} +3\epsilon_1+ \frac{Z_i\iota_1}{3n^k(s,a)} \right)
    \\ & \leq \frac{P^{k,\ddagger}_{s,z,a}v^2}{2} +\frac{4Z_i^2\iota_1}{3n^k(s,a)} +3\epsilon_{1}Z_i.
\end{align}
Then it follow that
\begin{align}
    \mathbb{V}(\hat{P}^{k,\ddagger}_{s,z,a},V_{h+1}^k) \leq \frac{3}{2}\mathbb{V}(P^k_{s,z,a},V_{h+1}^k)+ \frac{4Z_i^2\iota_1}{3n^k(s,a)}+3\epsilon_1 Z_i, 
\end{align}
and
\begin{align}
     b_h^k(s,z,a) + \sqrt{\frac{4\mathbb{V}(P^{k,\ddagger}_{s,z,a},V_{h+1}^{k}) \iota_1}{n^k(s,a)}} +\frac{Z_i\iota_1}{3n^k(s,a)} \leq 4\sqrt{\frac{ \mathbb{V}(P^{k,\ddagger}_{s,z,a},V_{h+1}^{k})  \iota_1}{n^k(s,a)}} +\frac{6Z_i\iota_1}{n^k(s,a)}+6\epsilon_{1}+\sqrt{12\epsilon_{1}Z_i}.\nonumber
\end{align}

Let $N_h^k(s,a)$ denote the visit count of $(s,a)$ before the $h$-th step in the $k$-th episode.
We define $\mathcal{K}$ be the set of indexes of episodes in which  update of empirical transition model does not occur. We further define $h_0(k) = \min\{h |N_h^k(s_h^k,a_h^k)+1\in \mathcal{L} \}$ for each $k\in \mathcal{K}^{C}$ and $\mathcal{B} = \{(k,h)| k\in \mathcal{K}^{C},h_0(k)+1\leq h \leq H\}$. In words, $h_0(k)$ denotes the time when the first update of empirical transition model occurs and $\mathcal{B}$ consists of $(k,h)$ pairs after such updates. We use $I(k,h)$ as a shorthand of the indicator $\mathbb{I}\left[ (k,h)\notin \mathcal{B} \right]$.

Define $\beta_h^k(s,z,a) =\min\{4\sqrt{\frac{ \mathbb{V}(P^{k,\ddagger}_{s,z,a},V_{h+1}^{k})  \iota_1}{n^k(s,a)}} +\frac{6Z_i\iota_1}{n^k(s,a)}+6\epsilon_{1}+\sqrt{12\epsilon_{1}Z_i}, Z_i \}$. By \eqref{eq:local1}, we have that
\begin{align}
  Q_h^k(s,z,a)I(k,h)\leq   \left(r^k(s,z,a) +P^k_{s,z,a} V^k_{h+1}   +\beta_h^k(s,z,a)\right) I(k,h+1) +Z_{i}\mathbb{I}\left[ I(k,h)\neq I(k,h+1)\right]\nonumber
 \\ \label{eq:local2}
\end{align}
holds for any $(s,z,a)\in\mathcal{S}\times [Z_i+1]\times\mathcal{A}$. 
By \eqref{eq:local2} and the fact $V^k_h(s_h^k,z_h^k) = Q^k_h(s_h^k,z_h^k,a_h^k)$, we obtain that
\begin{align}
&\sum_{k=1}^{T_0}V^{k*}_{1}(s_1,1)  - \sum_{k=1}^{T_0}\sum_{h=1}^H r^k_h \nonumber
\\ & \leq \sum_{k=1}^{T_0} V^{k}_1(s_1,1)I(k,1) - \sum_{k=1}^{T_0}\sum_{h=1}^H r^k_h I(k,h+1) \nonumber
\\ & \leq  \sum_{k=1}^{T_0}\sum_{h=1}^H ( P^{k,\ddagger}_{s_h^k,z_h^k,a_h^k } -\textbf{1}_{s_{h+1}^k,z_{h+1}^k  }) V_{h+1}^k\cdot I(k,h+1)+\sum_{k=1}^{T_0}\sum_{h=1}^H \beta_h^k(s_h^k,z_h^k,a_h^k)I(k,h+1)  +Z_i|\mathcal{K}^{C}|\label{eq:local3} .
\end{align}
Here \eqref{eq:local3} is by the fact that $\sum_{k=1}^{T_0}\sum_{h=1}^H \mathbb{I}\left[I(k,h)\neq I(k,h+1)\right]\leq |\mathcal{K}^C|$.
Define $M_1 = \sum_{k=1}^{T_0}\sum_{h=1}^H (P^k_{s_{h}^k, z_h^k,a_h^k}-\textbf{1}_{s_{h+1}^k , n_{h+1}^k}  )V_{h+1}^{k}\cdot I(k,h+1)$ and $M_2 = \sum_{k=1}^{T_0} \sum_{h=1}^H \beta_{h}^k(s_h^k,z_h^k,a_h^k)I(k,h+1)$. We will separately bound $M_1$ and $M_2$. 

\subsubsection{Bound of $M_1$}\label{sec:bdM1}
Define $\check{V}^k_{h} = I(k,h) V_{h}^k$. By definition of $\mathcal{B}$, $I(k,h+1)$ is measurable with respect to $\mathcal{F}_h^k$, where $\mathcal{F}_h^k=\sigma\left(\{s_{h'}^{k'},z_{h'}^{k'},a_{h'}^{k'},r_{h'}^{k'},s_{h'+1}^{k'}\}_{1\leq k'<k,1\leq h'\leq H} \cup \{s_{h'}^k,z_{h'}^k,a_{h'}^k,r_{h'}^k,s _{h'+1}^k\}_{1\leq h'\leq h-1}\right) $, i.e., all past trajectories before $(s_{h}^k,z_h^k,a_h^k$ is executed. Therefore, $M_1$ could be viewed as a martingale and we then have by Lemma \ref{self-norm} that
\begin{align}
    \mathbb{P}\left[ |M_1|> \sqrt{2\sum_{k=1}^{T_0}\sum_{h=1}^H \mathbb{V}\left( P^{k,\ddagger}_{s_h^k,z_h^k,a_h^k} , \check{V}_{h+1}^k \right)\iota }+ 6Z_i \iota  \right]\leq 2(\log_2(T_0H)+1)\delta. \label{eq:glo0}
\end{align}
We define $M_3 = \sum_{k=1}^{T_0}\sum_{h=1}^H \mathbb{V}\left(P^{k,\ddagger}_{s_h^k,z_h^k,a_h^k},\check{V}_{h+1}^k  \right)$ and deal with this term in Section \ref{sec:bdM2}.

\subsubsection{Bound of $M_2$}\label{sec:bdM2}

By the definition of $\beta_{h}^k(s,z,a)$, we have
\begin{align}
    M_2 &= \sum_{k=1}^{T_0}\sum_{h=1}^H \beta_{h}^k(s_{h}^k,z_h^k,a_h^k)I(k,h+1) \nonumber
    \\ & \leq \sum_{k=1}^{T_0}\sum_{h=1}^H 4\sqrt{\frac{ \mathbb{V}(P^{k,\ddagger}_{s_h^k,z_h^k,a_h^k} ,V_{h+1}^k )\iota_1 }{n^k(s,a)}} I(k,h+1) +\sum_{k=1}^{T_0}\sum_{h=1}^H 6\frac{ Z_i\iota_1}{n^k(s,a)}I(k,h+1)+ 6T_{0}H\epsilon_1 +\sqrt{12\epsilon_1Z_i}T_0H \nonumber
    \\ & = \sum_{k=1}^{T_0}\sum_{h=1}^H 4\sqrt{\frac{ \mathbb{V}(P^{k,\ddagger}_{s_h^k,z_h^k,a_h^k} ,\check{V}_{h+1}^k )\iota_1 }{n^k(s,a)}}  +\sum_{k=1}^{T_0}\sum_{h=1}^H 6\frac{ Z_i\iota_1}{n^k(s,a)}I(k,h+1)+ 6T_{0}H\epsilon_1 +\sqrt{12\epsilon_1Z_i}T_0H. \label{eq:local6}
\end{align}
Define $l_{i} = \left\lfloor \log_{2}(Z_i) \right\rfloor+1$.
By the update rule, for any  we have that
for any $(s,a)$ and any $ 3\leq j\leq l_i$, we have
 \begin{align}
     &  \sum_{k=1}^{T_0}\sum_{h=1}^H  \mathbb{I}\left[ (s_h^k,a_h^k)=(s,a), n^k(s,a)=2^{j-1}\right]\cdot I(k,h+1) \nonumber
     \\ & \leq  \sum_{k=1}^{T_0}\sum_{h=1}^H  \mathbb{I}\left[ (s_h^k,a_h^k)=(s,a), n^k(s,a)=2^{j-1}\right]\cdot I(k,h)  \nonumber
     \\ &\leq 2^{j-1}.\label{eq_sec3_a01}
 \end{align}
 We then obtain
 \begin{align}
     &\sum_{k=1}^{T_0} \sum_{h=1}^H\sqrt{ \frac{\mathbb{V}(P^{k,\ddagger}_{s_h^k,z_h^k,a_h^k} ,\check{V}_{h+1}^k ) }{n^k(s_h^k,a_h^k)  }} \nonumber
\\ &  \leq  \sum_{k=1}^{T_0}\sum_{h=1}^H   \sum_{s,a}\sum_{j=3}^{l_{i}}\mathbb{I}\left[ (s_h^k,a_h^k)=(s,a), n^k(s,a)=2^{j-1}\right]  \sqrt{\frac{\mathbb{V}(P^{k,\ddagger}_{s_h^k,z_h^k,a_h^k} ,\check{V}_{h+1}^k )}{2^{i-1}}}+  8SAZ_i \label{eq:local5}
\\ & = \sum_{s,a}\sum_{j=3}^{l_i} \frac{1}{\sqrt{2^{j-1}}} \sum_{k=1}^{T_0}\sum_{h=1}^H  \mathbb{I}\left[ (s_h^k,a_h^k)=(s,a), n^k(s,a)=2^{j-1}\right] I(k,h+1) \cdot \sqrt{\mathbb{V}(P^{k,\ddagger}_{s_h^k,z_h^k,a_h^k} ,\check{V}_{h+1}^k ) } + 8SAZ_i\nonumber
\\ &  \leq \sum_{s,a}\sum_{j=3}^{l_i} \sqrt{\frac{ \sum_{k=1}^{T_0}\sum_{h=1}^H  \mathbb{I}\left[ (s_h^k,a_h^k)=(s,a), n^k(s,a)=2^{j-1}\right] I(k,h+1) }{2^{j-1}}}\cdot \nonumber
\\  & \quad \quad \quad  \sqrt{  \left( \sum_{k=1}^{T_{0}}\sum_{h=1}^H  \mathbb{I}\left[ (s_h^k,a_h^k)=(s,a), n^k(s,a)=2^{j-1}\right]\mathbb{V}(P^{k,\ddagger}_{s_h^k,z_h^k,a_h^k} ,\check{V}_{h+1}^k ) \right)   } + 8SAZ_i\label{eq_sec3_01}
\\ & \leq \sqrt{ SAl_{i}\sum_{k=1}^{T_0} \sum_{h=1}^H \mathbb{V}(P^{k,\ddagger}_{s_h^k,z_h^k,a_h^k} ,\check{V}_{h+1}^k )  }+ 8SAZ_i \label{eq:local3.5}
 \end{align}
 Here \eqref{eq:local5} holds by bounding $ \sqrt{ \frac{\mathbb{V}(P^{k,\ddagger}_{s_h^k,z_h^k,a_h^k} ,\check{V}_{h+1}^k ) }{n^k(s_h^k,a_h^k)  }}$ by $Z_i$ for the $(k,h)$ pairs such that $n^k(s_h^k,a_h^k)\leq 4$, \eqref{eq_sec3_01} holds by Cauchy-Schwartz inequality and \eqref{eq:local3.5} is by \eqref{eq_sec3_a01}.
 
On the other hand, in a similar way we have that
\begin{align}
    \sum_{k=1}^{T_0}\sum_{h=1}^H \frac{I(k,h+1)}{n^k(s,a)}& \leq SAl_i +8SA.
     \label{eq:local4}
\end{align}

Recall that $M_3 = \sum_{k=1}^{T_0}\sum_{h=1}^H \mathbb{V}(P^{k,\ddagger}_{s_h^k,z_h^k,a_h^k} ,\check{V}_{h+1}^k)$.
By \eqref{eq:local3}, \eqref{eq:local6} and \eqref{eq:local4} , we obtain that
\begin{align}
M_2& \leq O\left(  \sqrt{SAl_i\iota_1 M_3}   +SAl_iZ_i+ SAZ_i\iota_1 + 6T_0H\epsilon_1 +\sqrt{12\epsilon_1}T_0H \right)
\\ & = O\left(  \sqrt{SAl_i\iota_1 M_3}   +SAl_iZ_i+ SAZ_i\iota_1 \right).\label{eq:glo1}
\end{align}

Now we deal with $M_3$ by a recursive-based concentration bound. 

Define $F(m) =  \sum_{k=1}^{T_0} \sum_{h=1}^H (P^{k,\ddagger}_{s_h^k,z_h^k,a_h^k}-\textbf{1}_{s_{h+1}^k,z_{h+1}^k}) \left(\frac{\check{V}_{h+1}^k}{Z_i}\right)^{2^{m}}$ for $m \geq 1$.
As argued before, $\check{V}_{h+1}^k$ is measurable with respect to $\mathcal{F}_h^k$.
By Lemma \ref{self-norm}, for each $1\leq m\leq \left \lfloor \log_{2}(T_0 H)\right \rfloor +1$,  it holds that
\begin{align}
    \mathbb{P}\left[ |F(m)|>2\sqrt{2\sum_{k=1}^{T_0}\sum_{h=1}^H\mathbb{V}\left(P^{k,\ddagger}_{s_h^k,z_h^k,a_h^k} ,\left(\frac{\check{V}_{h+1}^k}{Z_i}\right)^{2^m} \right) \iota }+ 6\iota \right]\leq 2(\log_{2}(KH)+1)\delta.
\end{align}
Direct computation gives that
\begin{align}
&  \sum_{k=1}^{T_0}\sum_{h=1}^H\mathbb{V}\left(P^{k,\ddagger}_{s_h^k,z_h^k,a_h^k} , (\check{V}_{h+1}^k)^{2^m} \right) = \sum_{k=1}^{T_0} \sum_{h=1}^H \left(   P^{k,\ddagger}_{s_h^k,z_h^k,a_h^k} (\check{V}_{h+1}^k)^{2^{m+1}} - \left(    P^{k,\ddagger}_{s_h^k,z_h^k,a_h^k} (\check{V}_{h+1}^k)^{2^m} \right)^2       \right)\nonumber
\\ &   = \sum_{k=1}^{T_0} \sum_{h=1}^H (P^{k,\ddagger}_{s_h^k,z_h^k,a_h^k}-\textbf{1}_{s_{h+1}^k,z_{h+1}^k}) (\check{V}_{h+1}^k)^{2^{m+1}} \nonumber
\\& \quad \quad \quad \quad +\sum_{k=1}^{T_0} \sum_{h=1}^H \left(  (\check{V}_{h}^k(s_h^k,z_h^k))^{2^{m+1}} -   \left(    P^{k,\ddagger}_{s_h^k,z_h^k,a_h^k}  (\check{V}_{h+1}^k))^{2^{m+1}}  \right)  \right)   - \sum_{k=1}^{T_0} (V_1^k(s_1^k,1))^{2^{m+1}}\nonumber
\\ & \leq  \sum_{k=1}^{T_0} \sum_{h=1}^H (P^{k,\ddagger}_{s_h^k,z_h^k,a_h^k}-\textbf{1}_{s_{h+1}^k,z_{h+1}^k}) (\check{V}_{h+1}^k)^{2^{m+1}} +\sum_{k=1}^{T_0} \sum_{h=1}^H \left(  (\check{V}_{h}^k(s_h^k,z_h^k))^{2^{m+1}} -   \left(    P^{k,\ddagger}_{s_h^k,z_h^k,a_h^k} \check{V}_{h+1}^k)^{2^{m+1}} \right)  \right)  \nonumber
\\ & \leq \sum_{k=1}^{T_0} \sum_{h=1}^H (P^{k,\ddagger}_{s_h^k,z_h^k,a_h^k}-\textbf{1}_{s_{h+1}^k,z_{h+1}^k}) (\check{V}_{h+1}^k)^{2^{m+1}}  +2^{m+1}Z_i^{2^{m+1}-1} \sum_{k=1}^{T_0} \sum_{h=1}^H \max\{    r^k(s_h^k,z_h^k,a_h^k) +\beta_h^k(s_h^k,z_h^k,a_h^k)  ,0\}I(k,h+1)\nonumber
\\ & \leq  \sum_{k=1}^{T_0} \sum_{h=1}^H (P^{k,\ddagger}_{s_h^k,z_h^k,a_h^k}-\textbf{1}_{s_{h+1}^k,z_{h+1}^k}) (\check{V}_{h+1}^k)^{2^{m+1}}+   2^{m+1}Z_i^{2^{m+1}-1} \sum_{k=1}^{T_0} \sum_{h=1}^H \beta_h^k(s_h^k,z_h^k,a_h^k)I(k,h+1)   + 2^{m+1}T_0Z_i^{2^{m+1}}\nonumber
\\ & =  \sum_{k=1}^{T_0} \sum_{h=1}^H (P^{k,\ddagger}_{s_h^k,z_h^k,a_h^k}-\textbf{1}_{s_{h+1}^k,z_{h+1}^k}) (\check{V}_{h+1}^k)^{2^{m+1}}+2^{m+1}Z_{i}^{2^{m+1}}( M_{2}/Z_i  +T_0).\label{eq:local7}
\end{align}

Therefore, for each $m\geq 1$ with probability $1-2(\log_2(T_0 H)+1)\delta$, it holds that
\begin{align}
  |F(m)|\leq 2\sqrt{2F(m+1) +2^{m+2}(M_2/Z_i+T_0) }
\end{align}
By Lemma \ref{lemma2} with $\lambda_1=T_0H$, $\lambda_2= 8\iota$, $\lambda_3= M_2/Z_i+T_0 +|\mathcal{K}|^{C}$ and $\lambda_4=6\iota$, we have that
\begin{align}
    \mathbb{P}\left[ |F(1)|> \max\{  46\iota, 8\sqrt{ (M_2/Z_i+T_0)\iota }+6\iota       \}  \right]\leq 2(\log_2(T_0H)+1)\log_2(T_0H)\delta.\label{eq:local8}
\end{align}
Plugging $m = 0$ into \eqref{eq:local7}, we have that
\begin{align}
    M_3 \leq Z_i^2F(1)+ 2(Z_i M_2+ T_0 Z_i^2).\nonumber
\end{align}
It then holds that 
\begin{align}
\mathbb{P}\left[ M_3 > 6(Z_i M_2+ T_0 Z_i^2) +46Z_i\iota \right]\leq 2(\log_2(T_0H)+1)\log_2(T_0H)\delta.\label{eq:glo2}
\end{align}
By \eqref{eq:glo1} and \eqref{eq:glo2}, with probability $1-2(\log_2(T_0H)+1)\log_2(T_0H)\delta$, it holds that
\begin{align}
   & M_2\leq  O\left(\sqrt{SAl_i\iota_1 M_3} + SA(\iota_1+l_i) Z_i \right) ;\nonumber
    \\& M_3 \leq 6(Z_i M_2+ T_0 Z_i^2) +46Z_i\iota ,\nonumber
\end{align}
which implies that $M_3 \leq O\left(Z_i \sqrt{SAT_0 l_i \iota_1} + SAl_i \iota_1 Z_i \right)$.

Combining this with \eqref{eq:glo0}, we finally conclude with probability $1-2(\log_2(T_0H)+1)\log_2(T_0 H)\delta +4SA(\log_2(Z_i)+2)\delta$, it holds that
\begin{align}
 \textbf{Regret}&\leq  O\left(Z_i \sqrt{SAT_0 l_i \iota_1} + SAl_i \iota_1 Z_i \right)
\\& = O\left(Z_i \sqrt{SAl_i( \iota + S\ln(SAH/\epsilon)) T_0}+ SAZ_il_i ( \iota + S\ln(SAH/\epsilon)) \right). \label{eq:add2}
\end{align}
 Here \eqref{eq:add2} holds because $T_0$ is bounded by $\mathrm{poly}(S,A,\ln(H),1/\epsilon)\iota$.
The proof is completed.

\section{Missing Proofs in Section \ref{sec:pf_plan}}\label{sec:mpf_plan}

\subsection{Proof of Lemma \ref{lemma:Qopt}}\label{sec:mpf_lemma_Qopt}
We will prove by backward induction from $h=H+1$. Firstly, the conclusion holds trivially for $h=H+1$. Assume $Q_{h'}(s,a)\geq Q_{h'}^{\dagger* }(s,a)$ for any $(s,a)$ and $h+1\leq h'\leq H+1$. By the update rule \eqref{eq:update} and the fact that $Q_h^{\dagger *}\leq 1$, we have that
\begin{align}
    &Q_h(s,a) -Q_h^{\dagger*}(s,a) \nonumber
    \\ & \geq \min\{b_h(s,a) +\hat{P}_{s,a}^{\dagger}V_{h+1}-P_{s,a}^{\dagger}V_{h+1}+ P_{s,a}^{\dagger}(V_{h+1}- V^{\dagger*}_{h+1})   ,     0\}\nonumber
    \\ & \geq \min\{b_h(s,a) +\hat{P}_{s,a}^{\dagger}V_{h+1}-P_{s,a}^{\dagger}V_{h+1},      0\}\nonumber
    \\ & \geq \min\{b_h(s,a)- \sum_{s'} \left(2(1-\frac{1}{Z_{i(s,a)}})\sqrt{\frac{\hat{P}^{\dagger}_{s,a}\iota  }{(1-\frac{1}{Z_{i(s,a)}})N(s,a)}} +\frac{14\iota}{3N(s,a)} \right) \cdot |V_{h+1}-\hat{P}_{s,a}^{\dagger}V_{h+1}|, 0             \}\nonumber
    \\ & \geq \min \{b_h(s,a) -2\sqrt{\frac{S\iota \mathbb{V}(\hat{P}^{\dagger}_{s,a} ,V_{h+1} ) }{N(s,a)}} - \frac{14S \iota }{3N(s,a)} \}\nonumber
    \\ & \geq 0.\nonumber
\end{align}

\section{Other Missing Proofs}\label{app:omf}
\subsection{Proof of Proposition \ref{pro1}}
Recall the definition of good event $\mathcal{G}$ in Section \ref{sec:mpf_plan}. We will prove conditioned on $\mathcal{G}$.

With a slight abuse of notation, we use $\{Q_h(s,a )\}_{(s,a,h)\in \mathcal{S}\times\mathcal{A}\times [H]}$ and $\{V_h(s)\}_{(s,h)\in \mathcal{S}\times [H]}$ to denote respectively the $Q$-function and $V$-function  returned by \textsc{Q-Computing}($\hat{P},N,r$). 
Following similar lines in the proof of Lemma \ref{lemma:Qopt}, we have that conditioned on $\mathcal{G}$, $Q_h(s,a)\geq Q^{*}(s,a)$ for any $(s,a,h)\in \mathcal{S}\times \mathcal{A}\times [H]$. By \eqref{eq:alg4_update1} and \eqref{eq:alg4_update2}, the Bellman error of the computed $Q$-function is bounded by 
\begin{align}
    Q_h(s,a)-r(s,a) -P_{s,a}V_{h+1} & \leq b_h(s,a) +(\hat{P}_{s,a}-P_{s,a})V_{h+1}\nonumber
    \\ & \leq b_h(s,a) +\sum_{s'}\left( \sqrt{\frac{2P_{s,a,s'}\iota}{N(s,a)}}+ \frac{\iota}{3N(s,a)} \right) \cdot \left| V_{h+1}(s')- P_{s,a}V_{h+1} \right| \label{eq:appc_local1}
    \\ & \leq b_h(s,a)+ \sqrt{\frac{2S\iota \mathbb{V}(P_{s,a},V_{h+1}) }{N(s,a)}}+ \frac{S\iota}{3N(s,a)}\label{eq:appc_local2}
    \\ & \leq 6\sqrt{\frac{S\iota \mathbb{V}(P_{s,a},V_{h+1} )}{N(s,a)}} +\frac{9S\iota}{N(s,a)}.\nonumber
\end{align}
Here \eqref{eq:appc_local1} holds by $|\hat{P}_{s,a,s'}-P_{s,a,s'} |\leq \sqrt{\frac{2\iota P_{s,a,s'}}{N(s,a)}}+ \frac{\iota}{3N(s,a)}$. On the other hand, $Q_h(s,a)\leq 1$ implies that $Q_h(s,a)-r(s,a)-P_{s,a}V_{h+1}\leq 1$. Re-define $\beta_h(s,a) = \min \{ 6\sqrt{\frac{S\iota \mathbb{V}(P_{s,a},V_{h+1} )}{N(s,a)}} +\frac{9S\iota}{N(s,a)}  ,1  \}$. It then follows
\begin{align}
Q_h(s,a)-r(s,a)-P_{s,a}V_{h+1}\leq \beta_h(s,a).\label{eq:appc_local3}
\end{align}
Let $\pi$ be the policy such that $\pi_h(s) = \arg\max_{a}Q_h(s,a)$ for any $(s,h)\in\mathcal{S}\times [H]$ and re-define $w_h(s,a,\pi): =\mathbb{E}_{\pi,\mathcal{M}}\left[ \mathbb{I}\left[(s_h,a_h)=(s,a) \right] \right]$. Then we have
\begin{align}
   & V_{1}^{*}(s_1)-V_{1}^{\pi}(s_1)\nonumber
    \\&\leq V_{1}(s_1)-V_{1}^{\pi}(s_{1})\nonumber
    \\ & \leq \beta_{1}(s_1,a_1) +P_{s_1,a_1}(V_{2}-V_{2}^{\pi})
    \\ & \leq ...\nonumber
    \\ & \leq \sum_{s,a,h}w_{h}(s,a,\pi)\beta_{h}(s,a)\nonumber
   \\& \leq  O \left(\sum_{(s,a)\notin\mathcal{X}_{K+1}}\sum_h w_h(s,a,\pi)\sqrt{\frac{S\iota \mathbb{V}(P_{s,a},V_{h+1} )}{N(s,a)}}  +\sum_{(s,a)\notin \mathcal{X}_{K+1}}\sum_h w_h(s,a,\pi)\frac{S\iota}{N(s,a)} \right )\nonumber
  \\ & \quad \quad \quad \quad  + \sum_{(s,a)\in \mathcal{X}_{K+1}}\sum_{h}w_h(s,a,\pi) \nonumber
   \\ & \leq\sum_{i=1}^K O\left( \sum_{(s,a)\in \mathcal{X}_i}\sum_{h}w_h(s,a,\pi)\sqrt{\frac{S\iota \mathbb{V}(P_{s,a},V_{h+1}) }{N_i}}  +\sum_{(s,a)\in \mathcal{X}_i}\frac{S\iota}{N_i} \right)  +O(\epsilon)\nonumber
   \\ & \leq \sum_{i=1}^K O\left( \sqrt{\frac{\sum_{h=1}^H w_h(s,a,\pi)}{N_i}} \cdot \sqrt{S\iota \sum_{h=1}^H w_h(s,a,\pi)\mathbb{V}(P_{s,a},V_{h+1})}  +\sum_{h=1}^H w_h(s,a,\pi)\cdot \frac{S\iota}{N_i} \right) +O(\epsilon)\nonumber
    \\ & \leq \sum_{i=1}^K O\left(\epsilon \sqrt{\sum_{h=1}^H w_h(s,a,\pi)\mathbb{V}(P_{s,a},V_{h+1})} +\epsilon^2 \right) + O(\epsilon)\nonumber
   \\ & \leq O(K\epsilon \sqrt{\sum_{h=1}^H w_h(s,a,\pi)\mathbb{V}(P_{s,a},V_{h+1})}+ K\epsilon^2+\epsilon).\label{eq:appc_local4}
\end{align}
Note that
\begin{align}
   & \sum_{h=1}^H w_h(s,a,\pi)\mathbb{V}(P_{s,a},V_{h+1})\nonumber
   \\ & = \mathbb{E}_{\pi,\mathcal{M}}\left[\sum_{h=1}^H \left( P_{s,a}(V_{h+1})^2 - (P_{s,a}V_{h+1})^2 \right)  \right] \nonumber
   \\ & \leq \mathbb{E}_{\pi,\mathcal{M}}\left[ \sum_{h=1}^H (V_h(s_h))^2 - (P_{s,a}V_{h+1})^2 \right] \nonumber
   \\ & \leq 2\mathbb{E}_{\pi,\mathcal{M}}\left[\sum_{h=1}^H (r(s_h,a_h)+\beta_h(s_h,a_h) )  \right]\nonumber
\\ & \leq 2+ 2\sum_{s,a,h}w_h(s,a,\pi)\beta_h(s,a).\label{eq:appc_local5}
\end{align}
By \eqref{eq:appc_local4} and \eqref{eq:appc_local5}, we have that
\begin{align}
    \sum_{s,a,h}w_h(s,a,\pi)\beta_h(s,a)\leq O\left(K\epsilon \sqrt{2+   \sum_{s,a,h}w_h(s,a,\pi)\beta_h(s,a)} +\epsilon  \right),
\end{align}
which implies that 
\begin{align}
    V_{1}^*(s_1)-V_{1}^{\pi}(s_1)\leq O(\sum_{s,a,h}w_h(s,a,\pi)\beta_h(s,a))\leq O(K\epsilon +K^2\epsilon^2) .
\end{align}
By rescaling $\epsilon$, we finish the proof.


\end{document}